\documentclass[reqno,twoside]{amsart}

\usepackage[utf8]{inputenc}
\usepackage[T1]{fontenc}
\usepackage[english]{babel}
\usepackage{amsmath,amssymb,amsfonts,amsthm}
\usepackage{hyperref}
\usepackage{microtype}
\usepackage{fullpage}
\usepackage{floatrow}
\usepackage{lmodern}
\usepackage{textgreek}
\usepackage{mathtools}
\usepackage{mathrsfs}
\usepackage{graphicx}
\usepackage{xcolor}
\colorlet{BLUE}{black}

\graphicspath{{../}{../Figures/}{Figures/}}
\usepackage{microtype}
\usepackage{multirow}
\usepackage{booktabs}
\usepackage{sidecap}

\DisableLigatures{encoding = *, family = * }

\newtheorem{theorem}{Theorem}[section]
\newtheorem{definition}[theorem]{Definition}

\newtheorem{proposition}[theorem]{Proposition}
\newtheorem{remark}[theorem]{Remark}

\title[]{Bi-HYCO: Bi-Objective Cooperative Learning for PDE Parameter Identification under Fragmented Observations}

\author{Umberto Biccari\textsuperscript{\,$\ast$}}  
\address{\textsuperscript{$\ast$}\, Chair of Computational Mathematics, DeustoTech, University of Deusto, Avenida de las Universidades 24, 48007 Bilbao, Basque Country, Spain
	\newline \indent \textsuperscript{$\dagger$}\, School of Mathematics and Statistics, Beijing Institute of Technology, 100081 Beijing, China
	\newline \indent \textsuperscript{$\ddagger$}\, Chair for dynamics, control, machine learning, and numerics (Alexander Von Humboldt-Professorship), Department of Mathematics, Friedrich-Alexander-Universit\"at Erlangen-N\"urnberg, 91058, Erlangen, Germany.
	\newline \indent \textsuperscript{$\S$}\, Universidad Aut\'onoma de Madrid, Departamento de Matem\'aticas, Ciudad Universitaria de Cantoblanco, 28049 Madrid, Spain.	
} 
\email{umberto.biccari@deusto.es}
\thanks{This project has received funding from the European Research Council (ERC) under the European Union's Horizon Europe research and innovation programme (grant agreement No.~101096251-CoDeFeL). This material is based upon work supported by the Air Force Office of Scientific Research under award number FA8655-22-1-7012. EZ was partially supported by the Alexander von Humboldt Professorship program; the European Union's Horizon Europe MSCA project ModConFlex (HORIZON-MSCA-2021-DN-01, project 101073558); the Transregio 154 Project Mathematical Modeling, Simulation and Optimization Using the Example of Gas Networks of the DFG; and SURE-AI: The Norwegian Centre for Sustainable, Risk-Averse, and Ethical AI, grant 357482, Research Council of Norway, and the Madrid Government-UAM 		Agreement for the Excellence of the University Research Staff in the context of the V PRICIT. UB and EZ were partially supported by Grant PID2023-146872OB-I00-DyCMaMod of MICIU (Spain) and by COST Action CA24122-Multiscale Stochastics, Patterns, and Analysis of Combinatorial Environments. UB, RM, and EZ were partially supported by COST Action CA24136-Interactions between Control Theory and Machine Learning. JC was partially supported by the Program of China Scholarship Council (Grant No.~202506030033).}

\author{Jun Chen \textsuperscript{\,$\dagger$}}
\email{chenjun\_bcbm@163.com}

\author{Roberto Morales\textsuperscript{\,$\ast$}}
\email{roberto.morales@deusto.es}

\author{Enrique Zuazua\textsuperscript{\,$\ast$\,$\ddagger$\,$\S$}}  
\email{enrique.zuazua@fau.de, enrique.zuazua@deusto.es, enrique.zuazua@uam.es}



\begin{document}

\begin{abstract}
Physical and synthetic models may describe complementary aspects of the same PDE-governed system while receiving different, possibly fragmented, observations. We propose Bi-Objective HYCO (Bi-HYCO), a cooperative framework that retains both representations and their local observational objectives while coupling their predicted states at unlabeled interaction points. These points contain no measurements and do not augment the data; they provide a communication mechanism in the common state space. The two criteria form a vector-valued objective, and weighted scalarizations provide computational realizations. For the deterministic shared-observation algorithm with fixed interaction points, we prove sufficient decrease and finite length of the whole alternating sequence, which converges to a mixed critical point under the stated Kurdyka-\L{}ojasiewicz-type assumptions. Elliptic transmission and two-dimensional Navier-Stokes experiments assess parameter and state reconstruction, noise and scalarization effects, and PINN/XPINN references. Ablations show that removing state interaction while retaining aggregation deteriorates parameter recovery in the tested configurations, particularly for Navier-Stokes.
\end{abstract}

\maketitle

\section{Introduction}

Scientific phenomena often admit several imperfect mathematical or computational representations. Physics-based models encode mechanistic structure, interpretable parameters, and equation-governed extrapolation, but may be incomplete, coarsely discretized, or misspecified. Synthetic models adapt flexibly to observations, but may offer less physical consistency, transparency, or reliability outside the observed regime \cite{montans2019data,quarteroni2025combining}. Neither representation is assumed to be the truth; each may contain complementary information about an underlying state \(u^\dagger\).

The information may also be incomplete and heterogeneous: the two representations may receive observations from different spatial regions, time intervals, sensors, modalities, or data-holding units \cite{mcmahan2017communication,yang2019federated}. Their datasets may coincide, overlap, or be disjoint. This leads to a fundamental question: \emph{how can complementary models cooperate when each has access to only a partial view of the same phenomenon?}

Scientific machine learning combines physical structure and observational information in several ways \cite{quarteroni2025combining,xu2022physics,yang2021bpinns}. Representative physics-informed approaches, including PINNs and domain-decomposed variants, place data and differential constraints within a neural representation and its training objective \cite{jagtap2020extended,moseley2023fbpinns,raissi2019physics}. Classical PDE-constrained parameter estimation instead retains a numerical forward model and optimizes its parameters using established reduced, all-at-once, adjoint, and regularization techniques \cite{giles2000introduction,gunzburger2003perspectives,hinze2009optimization,plessix2006review}. Hybrid-Cooperative Learning (HYCO) \cite{liverani2025hyco,liverani2026hyco} follows a different philosophy: it preserves distinct physical and synthetic representations and allows them to exchange information through their predicted states.

In HYCO, the two models interact by enforcing agreement between their predictions at selected points in their common state space \cite{liverani2025hyco,liverani2026hyco}. The points are unlabeled: they supply no measurements, but provide locations where information learned by one representation can influence the other. Original HYCO already introduced distinct physical and synthetic models, state-level interaction, alternating training, a game-theoretic interpretation and Nash-equilibrium results in simplified settings, and it conceptually permits heterogeneous observations. A complementary state-space analysis appears in \cite{zuazua2026coercivitygap}.

When the two representations are informed by different observations, their roles in the reconstruction are no longer naturally described by a single criterion. Each model is driven by its own data while cooperation requires consistency between their predicted states. This motivates treating the two model-specific objectives explicitly, while preserving the state-level interaction introduced by HYCO.

We therefore introduce \emph{Bi-Objective HYCO} (Bi-HYCO), in which the physical and synthetic components retain their respective observational objectives and are coupled through state agreement. The resulting vector formulation makes the trade-off between the two criteria explicit and provides a natural basis for weighted scalarization yielding practical algorithms \cite{Jahn2004,miettinen1999nonlinear}. 

The concrete realization studied in this paper is PDE-constrained parameter identification. Let \(\mathcal G\) be the admissible parameter set, \(\mathcal U\) the state space, and
\[
\mathcal S_{\rm phy}:\mathcal G\to\mathcal U,
\qquad
\gamma\mapsto u_{\rm phy}(\gamma),
\]
the parameter-to-state map defined by a numerical PDE solver. If \(\gamma^\dagger\) denotes the unknown physical parameter, observations are modeled abstractly as
\[
d^\delta=\mathcal O\mathcal S_{\rm phy}(\gamma^\dagger)+\varepsilon,
\]
where \(\mathcal O:\mathcal U\to\mathcal Y\) is an observation operator and \(\varepsilon\) is an observational perturbation. Such problems are natural testbeds for Bi-HYCO because observations may be incomplete, parameters are inferred indirectly, and the parameter-to-state and observation maps may lack injectivity or stability, motivating admissibility constraints and regularization \cite{alexanderian2026computational,banks1989estimation,chavent2010nonlinear,kirsch2021introduction,tarantola2005inverse}. 

We develop two implementation regimes. Under \emph{shared observations}, both components use the same dataset and a weighted scalarization is minimized through an alternating scheme. For this formulation, we establish convergence of the entire sequence to a mixed critical point under the assumptions of Section~\ref{section:convergence:result}, when the interaction points are fixed. Under \emph{fragmented observations}, instead, the components are informed by distinct datasets and operate through a decentralized scheme in which local parameter updates are combined by a coordinator while state-level interaction is maintained at unlabeled points. This formulation is inspired by Federated Averaging \cite{li2019convergence,mcmahan2017communication}; we provide an aggregate interpretation of its dynamics and assess its behavior numerically.

The numerical experiments examine two complementary parameter-identification problems. The first is a stationary linear elliptic transmission problem with a discontinuous coefficient, an irregular outer boundary, and spatially fragmented observations. The second is a nonlinear time-dependent Navier-Stokes problem on an annular domain, with viscosity and forcing amplitude as unknown parameters. Together they contrast stationary and evolutionary dynamics, linear and nonlinear models, and interface and annular geometries. The computations assess parameter recovery, state reconstruction, sensitivity to observational noise, scalarization weights, the effect of explicit state interaction, comparisons with PINN and XPINN reference formulations, and computational cost. 

The main contribution of this work is therefore a concrete cooperative methodology for PDE parameter identification under fragmented observations: distinct physical and synthetic models transfer information through agreement of their predicted states without centralizing their observation sets. Supporting contributions are:
\begin{itemize}
    \item an explicit bi-objective formulation and a precise Pareto/scalarization interpretation;
    \item whole-sequence convergence of the deterministic shared-observation alternating core; and
    \item a decentralized fragmented-observation realization, tested in two PDE problems through interaction ablations, scalarization analysis, and PINN/XPINN references.
\end{itemize}

Section~\ref{sec:methodology} presents the formulation and algorithms, Section~\ref{section:convergence:result} analyzes the deterministic core, Sections~\ref{sec:elliptic}-\ref{sec:annular-ns} report the PDE experiments, and Section~\ref{sec:conclusion} concludes. The Supplementary Materials collect ablations and implementation details.

\section{Bi-HYCO for PDE Parameter Identification with Fragmented Observations}\label{sec:methodology}

\subsection{Physical and synthetic representations and fragmented observations}

This section develops Bi-HYCO with a solver-based physical state and a parametric synthetic state in a common state space \(\mathcal U\). Their observation sets may be shared, partially overlapping, or disjoint; we use \emph{fragmented observations} for information held separately.

Let \(Q=\Omega\) for a stationary problem and
\(Q=\Omega\times(0,T)\) for a time-dependent one, where
\(\Omega\subset\mathbb{R}^d\) is open and \(T>0\). Let
\(\mathcal{U}\) be an appropriate state space over \(Q\).

\medskip 
\noindent {\bf Physical component.}
Let \(\gamma\in\mathcal{G}\subseteq\mathbb{R}^{d_{\mathcal G}}\) denote the
physical parameter to be identified. We assume that the governing PDE,
together with its boundary and, when needed, initial conditions, defines a
numerical parameter-to-state map
\[
    \mathcal{S}_{\rm phy}:\mathcal{G}\rightarrow\mathcal{U},
    \qquad
    \gamma\mapsto u_{\rm phy}(\cdot;\gamma).
\]
The map \(\mathcal{S}_{\rm phy}\) is understood as the realized PDE solver
used in the reconstruction.

\medskip
\noindent {\bf Synthetic component.}
The second state representation is parametrized by
\(\theta\in\Theta\subseteq\mathbb{R}^{d_\Theta}\) and is written
\[
    \mathcal{S}_{\rm syn}:\Theta\rightarrow\mathcal{U},
    \qquad
    \theta\mapsto u_{\rm syn}(\cdot;\theta).
\]

The analysis below does not require a particular realization of
\(\mathcal{S}_{\rm syn}\). In the numerical examples it is represented by a
finite-dimensional neural ansatz.

\medskip
\noindent {\bf Observation operators.}
To distinguish the information assigned to the two components, we introduce
observation operators
\[
    \mathcal{O}_{\rm phy}:\mathcal{U}\to\mathcal{Y}_{\rm phy},
    \qquad
    \mathcal{O}_{\rm syn}:\mathcal{U}\to\mathcal{Y}_{\rm syn}.
\]
If \(u^\dagger\) is the state generating the observations, the data
may be written abstractly as
\[
    d_{\rm phy}^{\delta}
    =
    \mathcal{O}_{\rm phy}u^\dagger+\varepsilon_{\rm phy},
    \qquad
    d_{\rm syn}^{\delta}
    =
    \mathcal{O}_{\rm syn}u^\dagger+\varepsilon_{\rm syn}.
\]

These operators may coincide, overlap only through part of their information,
or probe disjoint regions or time intervals.

For the pointwise observations used in Sections~\ref{sec:elliptic} and
\ref{sec:annular-ns}, we consider physical and synthetic datasets
\[
    \mathcal{D}_{\rm phy}
    :=
    \{(z_i^{\rm phy},d_i^{\rm phy})\}_{i=1}^{N_{\rm phy}},
    \qquad
    \mathcal{D}_{\rm syn}
    :=
    \{(z_j^{\rm syn},d_j^{\rm syn})\}_{j=1}^{N_{\rm syn}}.
\]
with corresponding local losses
\begin{align}
    \label{eq:L:phy}
    \mathcal{L}_{\rm phy}(\gamma)
    &:=
    \frac{1}{N_{\rm phy}}
    \sum_{i=1}^{N_{\rm phy}}
    \left|
    \mathcal{S}_{\rm phy}(\gamma)(z_i^{\rm phy})
    -
    d_i^{\rm phy}
    \right|^2
    +
    \mathcal{R}_{\rm phy}(\gamma),
\end{align}
and
\begin{align}
    \label{eq:L:syn}
    \mathcal{L}_{\rm syn}(\theta)
    &:=
    \frac{1}{N_{\rm syn}}
    \sum_{j=1}^{N_{\rm syn}}
    \left|
    \mathcal{S}_{\rm syn}(\theta)(z_j^{\rm syn})
    -
    d_j^{\rm syn}
    \right|^2
    +
    \mathcal{R}_{\rm syn}(\theta),
\end{align}
where \(\mathcal{R}_{\rm phy}\) and \(\mathcal{R}_{\rm syn}\) are optional
regularization terms.

\subsection{Unlabeled state interaction}

The two components exchange information through their predicted states. We organize this interaction into \emph{communication rounds}, indexed by \(m\), each corresponding to one stage of the cooperative procedure in which the current physical and synthetic states are compared at a set of common locations. 

At communication round \(m\), let \(\mathcal{Z}^m\coloneqq\{z_h^m\}_{h=1}^{H}\subset Q\) be a collection of unlabeled interaction points. The state discrepancy is
measured by
\begin{align}
    \label{eq:L:int}
    \mathcal{L}_{\rm int}^m(\gamma,\theta)
    &:=
    \frac{1}{H}
    \sum_{h=1}^{H}
    \left|
    \mathcal{S}_{\rm phy}(\gamma)(z_h^m)
    -
    \mathcal{S}_{\rm syn}(\theta)(z_h^m)
    \right|^2.
\end{align}

When \(\mathcal{Z}^m\) is fixed, the superscript \(m\) is omitted. The
interaction points carry no additional measured values; they only compare the
two reconstructed states. Consequently, increasing \(H\) does not amount to
adding physical observations.

\subsection{Bi-objective formulation, scalarization, and Pareto interpretation}

Let
\(\mathcal{X}:=\mathcal{G}\times\Theta\).
For a fixed interaction weight \(\lambda_{\rm int}\geq0\), define
\[
    F_{\rm phy}(\gamma,\theta)
    :=
    \alpha_{\rm phy}\mathcal{L}_{\rm phy}(\gamma)
    +
    \lambda_{\rm int}\mathcal{L}_{\rm int}(\gamma,\theta),
\]
and
\[
    F_{\rm syn}(\gamma,\theta)
    :=
    \alpha_{\rm syn}\mathcal{L}_{\rm syn}(\theta)
    +
    \lambda_{\rm int}\mathcal{L}_{\rm int}(\gamma,\theta),
\]
with \(\alpha_{\rm phy},\alpha_{\rm syn}>0\). Within the PDE-constrained parameter-identification setting considered here, the cooperative bi-objective problem is written as
\begin{equation}\label{eq:generic-biobjective-hyco}
    \min_{(\gamma,\theta)\in\mathcal{X}}
    \Big(
        F_{\rm phy}(\gamma,\theta),
        F_{\rm syn}(\gamma,\theta)
    \Big).
\end{equation}

The purpose of \eqref{eq:generic-biobjective-hyco} is to keep the two
observation-specific criteria explicit. The physical criterion depends on
\(\mathcal{O}_{\rm phy}\), the synthetic criterion on
\(\mathcal{O}_{\rm syn}\), and both contain the same state-interaction
mechanism. 

The optimality of \eqref{eq:generic-biobjective-hyco} is understood in the sense of Pareto. We recall that \((\gamma^\star,\theta^\star)\in\mathcal{X}\) is Pareto optimal if
there is no \((\gamma,\theta)\in\mathcal{X}\) such that
\[
    F_{\rm phy}(\gamma,\theta)
    \leq
    F_{\rm phy}(\gamma^\star,\theta^\star),
    \qquad
    F_{\rm syn}(\gamma,\theta)
    \leq
    F_{\rm syn}(\gamma^\star,\theta^\star),
\]
with at least one strict inequality. We refer to
\cite{Jahn2004,miettinen1999nonlinear} for the standard vector-optimization
terminology.

For the computational scheme and the subsequent convergence analysis, we consider the weighted scalarization
\begin{equation}\label{eq:generic-scalarization}
    \min_{(\gamma,\theta)\in\mathcal{X}}
    J(\gamma,\theta)
    :=
    \omega_{\rm phy}F_{\rm phy}(\gamma,\theta)
    +
    \omega_{\rm syn}F_{\rm syn}(\gamma,\theta),
\end{equation}
where $\omega_{\rm phy},\omega_{\rm syn}\geq0$ and $\omega_{\rm phy}+\omega_{\rm syn}=1$.

\begin{proposition}[Weighted scalarization]\label{prop:scalarization}
If \(\omega_{\rm phy},\omega_{\rm syn}>0\), every global minimizer of \eqref{eq:generic-scalarization} is Pareto optimal for \eqref{eq:generic-biobjective-hyco}. With merely nonnegative weights, a global minimizer is in general only weakly Pareto efficient.
\end{proposition}
\begin{proof}
If a positive-weight minimizer were dominated, multiplying the two componentwise inequalities by the weights and adding them would strictly decrease \(J\), a contradiction. If a weight vanishes, strict improvement of the unweighted objective need not change \(J\), which precludes the stronger conclusion.
\end{proof}

\begin{remark}
The result concerns global scalarized minima. Theorem~\ref{thm:Alternating:Algorithm} establishes mixed criticality, not global minimization, and hence does not imply Pareto optimality. Moreover, nonconvex weighted sums may miss unsupported Pareto points \cite{miettinen1999nonlinear,Jahn2004}.
\end{remark}

\subsection{Shared-observation alternating core}

We distinguish two computational regimes.

\medskip
\noindent
{\it Shared observations.}
Assume first that
\(\mathcal{D}_{\rm phy}=\mathcal{D}_{\rm syn}=:\mathcal{D}\).
Then
\begin{align}
    \nonumber 
    J(\gamma,\theta)
    =&
    \omega_{\rm phy}F_{\rm phy}(\gamma,\theta)
    +
    \omega_{\rm syn}F_{\rm syn}(\gamma,\theta)\\ \label{def:J}   
    =&
    a_{\rm phy}\mathcal{L}_{\rm phy}(\gamma)
    +
    a_{\rm syn}\mathcal{L}_{\rm syn}(\theta)
    +
    \lambda_{\rm int}\mathcal{L}_{\rm int}(\gamma,\theta),
\end{align}
where
\(a_{\rm phy}:=\omega_{\rm phy}\alpha_{\rm phy}\) and
\(a_{\rm syn}:=\omega_{\rm syn}\alpha_{\rm syn}\).
We minimize \eqref{def:J} by alternating the physical and synthetic
parameter blocks, as represented in Figure~\ref{fig:alternating:algorithm}.

\begin{figure}[!h]
    \centering
    \includegraphics[width=0.95\linewidth]{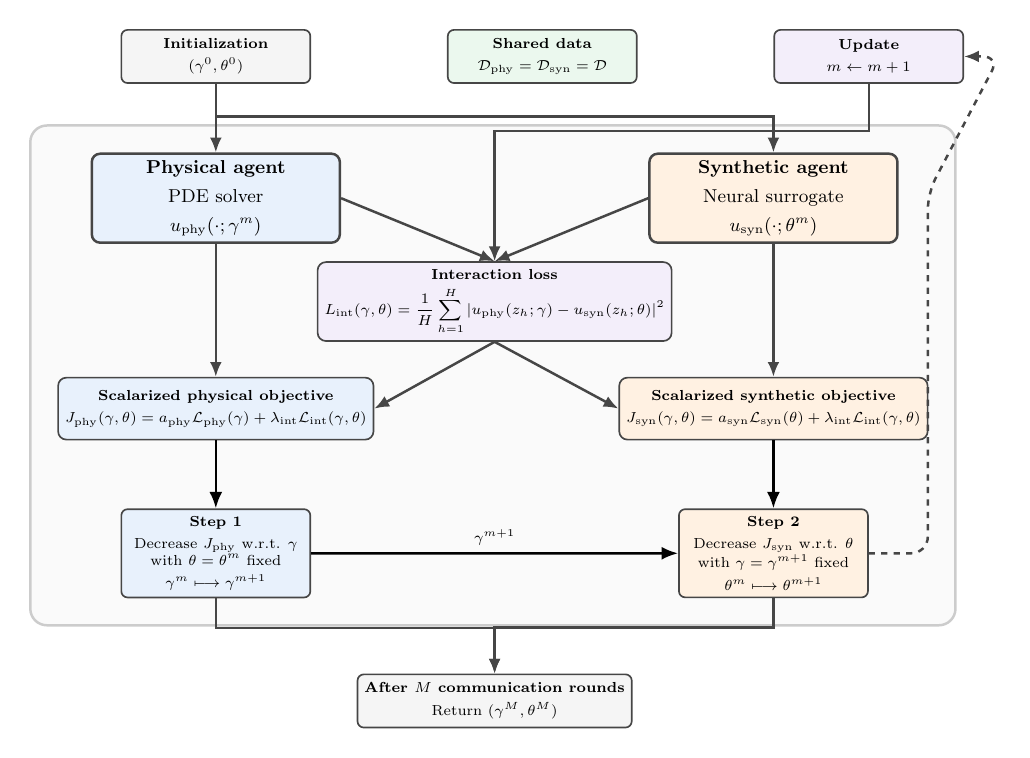}
    \caption{Alternating minimization of the scalarized Bi-HYCO functional in the
    shared-observation setting.}
    \label{fig:alternating:algorithm}
\end{figure}

If
\(\lambda_{\rm int}=0\) and \(a_{\rm phy},a_{\rm syn}>0\), the two blocks
decouple and the method reduces to
\begin{align}
    \label{eq:separated:problems}
    \min_{\gamma\in\mathcal{G}}\mathcal{L}_{\rm phy}(\gamma),
    \qquad
    \min_{\theta\in\Theta}\mathcal{L}_{\rm syn}(\theta).
\end{align}

\subsection{Fragmented-observation decentralized realization}
We next consider the case
\(\mathcal{D}_{\rm phy}\neq\mathcal{D}_{\rm syn}\). In this setting, the centralized formulation \eqref{def:J} is no longer directly available, since the two data-dependent losses are evaluated from datasets that are held separately. We therefore introduce a decentralized realization of the scalarized problem in which the observational information remains distributed.

More precisely, we consider two local computational units, each maintaining a copy of the parameter pair
\((\gamma,\theta)\). Unit \(1\) uses \(\mathcal{D}_{\rm phy}\), whereas unit \(2\)
has access to \(\mathcal{D}_{\rm syn}\). 

Using $a_{\rm phy}\coloneqq \omega_{\rm phy}\alpha_{\rm phy}$ and $a_{\rm syn}\coloneqq\omega_{\rm syn}\alpha_{\rm syn}$,
as in the shared-observation setting, the local objectives at communication round \(m\) are
\begin{align*}
    & J_1^m(\gamma,\theta) \coloneqq a_{\rm phy}\mathcal{L}_{\rm phy}(\gamma) +\lambda_{\rm int}\mathcal{L}_{\rm int}^m(\gamma,\theta), 
    \\
    & J_2^m(\gamma,\theta) \coloneqq a_{\rm syn}\mathcal{L}_{\rm syn}(\theta) +\lambda_{\rm int}\mathcal{L}_{\rm int}^m(\gamma,\theta).
\end{align*}

Starting from a common pair \((\gamma^m,\theta^m)\) provided by a coordinator, each unit performs \(E\) local updates of both parameter blocks according to \(J_i^m\). Although only one block has a direct data term on each unit, \(\mathcal L_{\rm int}^m\) couples both blocks and transmits an update signal to the other block. Both local parameter copies are therefore returned and aggregated as \(x^{m+1}=w_1x_1^{m,E}+w_2x_2^{m,E}\), where \(x=(\gamma,\theta)\), \(w_i\geq0\), and \(w_1+w_2=1\).

The complete procedure is summarized in Figure~\ref{fig:decentralized:algorithm}. This combination of local optimization and parameter aggregation is inspired by decentralized and federated optimization schemes \cite{li2019convergence,mcmahan2017communication}; here, the distinctive feature is the explicit state-level interaction between the physical and synthetic representations.

\begin{figure}[!h]
    \centering
    \includegraphics[width=1.00\linewidth]{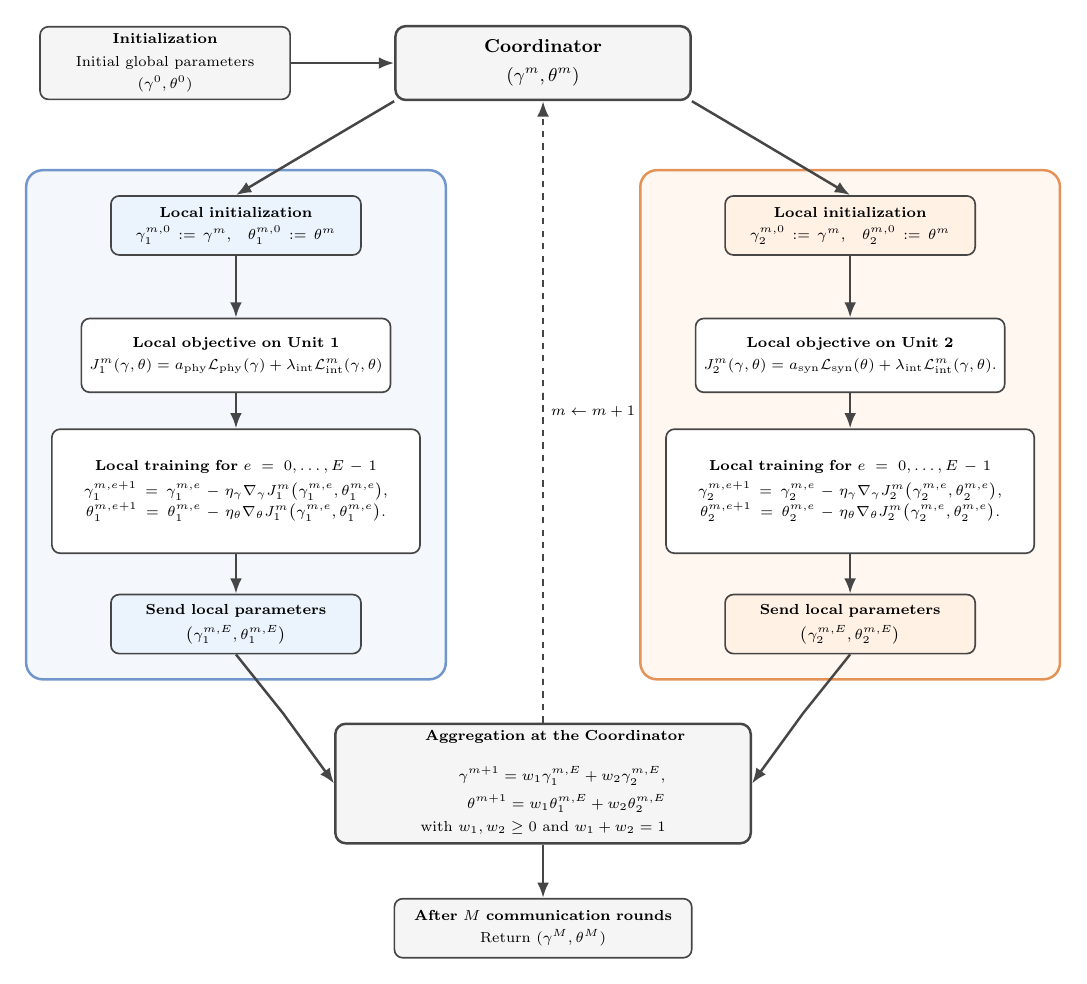}
    \caption{Decentralized realization for fragmented observations. Each local unit updates a copy of the parameter pair using its available observations and the interaction term. The coordinator then averages the updated parameter
    pairs. When \(\lambda_{\rm int}=0\), the explicit state-interaction term
    is absent, whereas coordinator aggregation remains active.}
    \label{fig:decentralized:algorithm}
\end{figure}

Setting \(\lambda_{\rm int}=0\) removes this state-level coupling while leaving the fragmented-data architecture, local optimization, and coordinator aggregation unchanged. This provides an internal reference for assessing the contribution of the interaction mechanism within the same decentralized reconstruction procedure.

The fragmented-observation scheme is a computational extension of the shared-observation formulation and is not covered by the convergence result of Section~\ref{section:convergence:result}. The latter concerns the deterministic shared-observation setting with fixed interaction points. The behavior of the decentralized scheme under fragmented observations is instead investigated numerically in Sections~\ref{sec:elliptic} and~\ref{sec:annular-ns}.

\subsection{Aggregate interpretation}

The decentralized scheme admits a useful interpretation in relation to Federated Averaging (FedAvg) \cite{li2019convergence,mcmahan2017communication}. At 
communication round $m$, both units start from the parameter pair \(x^m=(\gamma^m,\theta^m)\), perform $E$ local updates according to their respective objectives $J_i^m$, and return $x_i^{m,E}$. The coordinator then forms $x^{m+1}=w_1x_1^{m,E}+w_2x_2^{m,E}$, with $w_i\geq0$ and $w_1+w_2=1$.

This has the parameter-aggregation structure of FedAvg, with the additional feature that the local objectives are coupled through the common state-interaction 
term. To make this connection explicit, define the round-dependent aggregate functional
\[
    \overline J^{\,m}
    :=w_1J_1^m+w_2J_2^m
    =w_1a_{\rm phy}\mathcal L_{\rm phy}
     +w_2a_{\rm syn}\mathcal L_{\rm syn}
     +\lambda_{\rm int}\mathcal L_{\rm int}^m .
\]

Thus, although the observations remain distributed, the aggregate functional combines both data-dependent losses with the state-level coupling between the 
physical and synthetic representations.

The relation is exact when $E=1$. Indeed, if both units perform one gradient step with a common step size $\eta$, namely $x_i^{m,1}=x^m-\eta\nabla J_i^m(x^m)$, then aggregation gives
\[
    x^{m+1}
    =x^m-\eta\sum_{i=1}^2w_i\nabla J_i^m(x^m)
    =x^m-\eta\nabla\overline J^{\,m}(x^m).
\]

Hence, one local update followed by aggregation coincides with a gradient step for the aggregate functional.

For $E>1$, instead, the local iterates evolve along different trajectories before being aggregated. The coordinator update can then be viewed schematically as
\[
    x^{m+1}
    =x^m-\eta E\nabla\overline J^{\,m}(x^m)+r_{\rm drift}^m,
\]
where $r_{\rm drift}^m$ accounts for the discrepancy generated by evaluating the local gradients along different local trajectories. This is the usual client-drift effect associated with FedAvg \cite{li2019convergence}.

This interpretation is heuristic. A rigorous convergence analysis can be developed following standard arguments for FedAvg-type methods \cite{li2019convergence}, suitably adapted to account for the state-interaction term, and is omitted here.

\section{Convergence analysis of the alternating scheme}
\label{section:convergence:result}

This section studies the convergence of the scalarized alternating method represented in
Figure~\ref{fig:alternating:algorithm}. Here we consider unconstrained synthetic parameters
and set \(\Theta=\mathbb{R}^{d_{\Theta}}\). Accordingly, only the physical
variable is subject to an explicit constraint set.

Throughout this section, $\lambda_{\rm int}\geq 0$ is fixed. Consider $J$ defined as \eqref{def:J} and define the extended energy 
\begin{align}
    \label{def:extended:energy:Psi}
\Psi(\gamma,\theta):=J(\gamma,\theta) + i_{\mathcal{G}}(\gamma),
\end{align}
where $i_{\mathcal{G}}$ is the indicator function of $\mathcal{G}$, given by
\begin{align*}
    i_{\mathcal{G}}(\gamma)=0 \text{ if }\gamma\in \mathcal{G},\quad i_{\mathcal{G}}(\gamma)=\infty \text{ if }\gamma\notin \mathcal{G}.
\end{align*}

Moreover, since some activation functions may lead to non-smooth losses, in our analysis we use Clarke's generalized subdifferential \cite{Clarke1990}.

\begin{definition}
For a locally Lipschitz function $\varphi:\mathbb{R}^d\to \mathbb{R}$, the Clarke directional derivative at $x$ in the direction $v$ is 
\begin{align*}
    \varphi^\circ(x;v):=\limsup_{y\to x,\, t\searrow 0} \dfrac{\varphi (y+tv) - \varphi (y)}{t}. 
\end{align*}
In addition, the Clarke subdifferential is given by 
\begin{align*}
    \partial^\circ \varphi(x):=\{\xi \in \mathbb{R}^d\,:\, \varphi^\circ(x;v)\geq \langle \xi,v\rangle,\, \forall\, v\in \mathbb{R}^d \}.
\end{align*}
For each $\gamma\in \mathbb{R}^{d_{\mathcal{G}}}$ fixed, we set 
$J_\gamma(\theta)\coloneqq J(\gamma,\theta)$ for all $\theta\in \mathbb{R}^{d_\Theta}$. Then, we define the partial Clarke subdifferential with respect to $\theta$ as 
\begin{align*}
    \partial_{\theta}^\circ J(\gamma,\theta):=\partial^\circ J_\gamma (\theta),\quad \forall \, (\gamma,\theta) \in \mathbb{R}^{d_{\mathcal{G}}}\times \mathbb{R}^{d_{\Theta}}.
\end{align*}
\end{definition}

\begin{definition}%[Mixed subdifferential and mixed critical points]
For \((\gamma,\theta)\in \mathcal{G}\times\mathbb R^{d_\Theta}\), we define the mixed
subdifferential of the extended energy $\Psi$ defined in \eqref{def:extended:energy:Psi} by
\begin{align}
\label{def:partial:Psi:mix}
\partial_{\rm mix}\Psi(\gamma,\theta)
\coloneqq
\bigl(\nabla_\gamma J(\gamma,\theta)+N_{\mathcal{G}}(\gamma)\bigr)
\times
\partial^\circ_\theta J(\gamma,\theta),
\end{align}
where, for $\bar\gamma\in \mathcal{G}$, 
\begin{align*}
N_\mathcal{G}(\bar\gamma)
\coloneqq 
\left\{
\nu\in\mathbb R^{d_{\mathcal{G}}}:
\langle \nu,\gamma-\bar\gamma\rangle\leq 0
\quad\text{for all }\gamma\in \mathcal{G}
\right\}
\end{align*}
denotes the normal cone to $\mathcal{G}$ at $\bar\gamma$. Equivalently, $\xi\in\partial_{\rm mix}\Psi(\gamma,\theta)$ if and only if $\xi=(\xi_\gamma,\xi_\theta)$, with
\begin{align*}
    \xi_\gamma\in \nabla_\gamma J(\gamma,\theta)+N_{\mathcal{G}}(\gamma),
\qquad
\xi_\theta\in\partial^\circ_\theta J(\gamma,\theta).    
\end{align*}
A point $(\gamma^\star,\theta^\star)\in \mathcal{G}\times\mathbb R^{d_\Theta}$ is called a mixed critical point of $\Psi$ if $0\in\partial_{\rm mix}\Psi(\gamma^\star,\theta^\star)$, i.e., 
\begin{align*}
    0\in \nabla_\gamma J(\gamma^\star,\theta^\star)+N_\mathcal{G}(\gamma^\star) \quad\text{ and }\quad 0\in \partial^\circ_\theta J(\gamma^\star,\theta^\star).
\end{align*}
\end{definition}

The analysis below concerns the scalarized shared-observation Bi-HYCO iteration with fixed interaction points. Given $(\gamma^m,\theta^m)$, the physical variable is updated by
\begin{align*}
    \gamma^{m+1}:=\text{Proj}_{\mathcal{G}} (\gamma^m - \eta_\gamma \nabla_\gamma J(\gamma^m,\theta^m)),
\end{align*}
and the synthetic variable is updated by applying a descent method to 
\begin{align}
    \label{def:Gm}
    G^m(\theta):=a_{\rm syn}\mathcal{L}_{\rm syn}(\theta) + \lambda_{\rm int}\mathcal{L}_{\rm int}(\gamma^{m+1},\theta).
\end{align}

Following the general strategy developed for nonconvex and nonsmooth descent methods \cite{AttouchBolteSvaiter2013,BolteSabachTeboulle2014,BolteDaniilidisLewis2007}, in this paper, we use the following version of the Kurdyka-\L{}ojasiewicz property to obtain convergence from sufficient decrease and relative error estimates 

\begin{definition}
Let $\Psi:\mathbb {R}^{d_{\mathcal{G}}}\times\mathbb R^{d_\Theta}
\to \mathbb R\cup\{+\infty\}$ be proper and lower semicontinuous. We say that \(\Psi\) satisfies the
Kurdyka-{\L}ojasiewicz property at $(\gamma^\star,\theta^\star)\in \mathcal{G}\times\mathbb R^{d_\Theta}$, with respect to \(\partial_{\rm mix}\Psi\), if there exist \(\eta>0\), a neighborhood \(U\)
of $(\gamma^\star,\theta^\star)$, and a continuous concave function
\(\varphi:[0,\eta)\to\mathbb R_+\) such that
\begin{align}
    \label{KL-01}
    \varphi(0)=0,\qquad
    \varphi\in C^1(0,\eta),\qquad
    \varphi'(s)>0
    \quad\forall \,s\in(0,\eta),
\end{align}
and
\begin{align}
\label{KL-02}
\varphi'\bigl(\Psi(\gamma,\theta)-\Psi(\gamma^\star,\theta^\star)\bigr)
\operatorname{dist}\bigl(0,\partial_{\rm mix}\Psi(\gamma,\theta)\bigr)
\geq 1
\end{align}
for all \((\gamma,\theta)\in U\) satisfying
\[
\Psi(\gamma^\star,\theta^\star)<\Psi(\gamma,\theta)<\Psi(\gamma^\star,\theta^\star)+\eta.
\]
We say that \(\Psi\) is a KL function, with respect to
\(\partial_{\rm mix}\Psi\), if it satisfies the above property at every mixed
critical point in its domain.
\end{definition}

\noindent Moreover, we make the following assumptions.
\begin{itemize}
    \item[\bf (A1)] Let $\mathcal{G}\subset \mathbb{R}^{d_{\mathcal{G}}}$ be a nonempty, closed and convex set. Suppose that $\mathcal{L}_{\rm phy}\in C^1$ on a neighborhood of $\mathcal{G}$. In addition, assume that $\mathcal{L}_{\rm syn}$ is proper, bounded from below, and locally Lipschitz. Assume also that $\mathcal{L}_{\rm int}$ is locally Lipschitz with respect to $\theta$, and $C^1$ with respect to $\gamma$. We also assume that $J$ is locally Lipschitz on bounded subsets of $G\times \mathbb{R}^{d_{\Theta}}$, and the Clarke subdifferential mapping $(\gamma,\theta)\mapsto \partial_\theta^\circ J(\gamma,\theta)$ is outer semicontinuous on the initial sublevel set. Moreover, \(\Psi\) is proper, lower semicontinuous, and satisfies the Kurdyka-{\L}ojasiewicz property with respect to \(\partial_{\rm mix}\Psi\).
    \item[\bf (A2)] The initial sublevel set 
    \begin{align*}
        L_0:=\{ (\gamma,\theta)\in \mathcal{G}\times \mathbb{R}^{d_{\Theta}}\,:\, \Psi(\gamma,\theta)\leq \Psi(\gamma^0,\theta^0)\}
    \end{align*}
    is bounded.
    \item[\bf (A3)] There exist constants $\ell_{\mathcal{G}},\ell_{\mathcal{G}\Theta}>0$ such that, on $L_0$, we have
    \begin{align*}
        \|\nabla_{\gamma} J(\gamma,\theta) -\nabla_{\gamma} J(\hat{\gamma},\theta) \| \leq \ell_{\mathcal{G}} \|\gamma-\hat{\gamma}\|,\quad \forall \,(\gamma,\theta),(\hat{\gamma},\theta)\in \mathcal{G}\times \mathbb{R}^{d_{\Theta}},\\
        \lambda_{\rm int}\| \nabla_{\gamma} \mathcal{L}_{\rm int}(\gamma,\theta) - \nabla_{\gamma} \mathcal{L}_{\rm int}(\gamma,\hat{\theta}) \|\leq \ell_{\mathcal{G}\Theta} \|\theta-\hat{\theta}\|,\quad \forall\, (\gamma,\theta),(\gamma,\hat{\theta})\in \mathcal{G}\times \mathbb{R}^{d_{\Theta}}.
    \end{align*}
    \item[\bf (A4)] The physical step-size satisfies $0<\eta_\gamma < 2/\ell_{\mathcal{G}}$.
    \item[\bf (A5)] There exist $a_1,a_2>0$ such that 
    \begin{align*}
        G^m(\theta^{m+1})\leq G^m(\theta^m)-a_1\|\theta^{m+1} - \theta^m\|^2,\quad \forall\, m\in \mathbb{N}^*,\\
        \operatorname{dist}(0,\partial^\circ G^m(\theta^{m+1}))\leq a_2 \|\theta^{m+1} - \theta^m\|,\quad \forall m\in \mathbb{N}^*.
    \end{align*}    
\end{itemize}

We then have the following convergence result for the scalarized Alternating Algorithm in the shared-observation setting.

\begin{theorem}
    \label{thm:Alternating:Algorithm}
    Assume {\bf (A1)-(A5)}. Let $\{(\gamma^m,\theta^m)\}_{m\geq 0}$ be generated by the Alternating Algorithm. Then,
    \begin{itemize}
        \item There exists $c>0$ such that 
        \begin{align}
            \label{eq:T1.0}
            \Psi(\gamma^m,\theta^m) - \Psi(\gamma^{m+1},\theta^{m+1}) \geq c\left( \|\gamma^{m+1}-\gamma^{m}\|^2 + \|\theta^{m+1}-\theta^m\|^2 \right).
        \end{align}

        In particular, 
        \begin{align}
            \label{eq:T1.1}
            \sum_{m=0}^{\infty} \|\gamma^{m+1}-\gamma^m\|^2 <\infty\text{ and }\sum_{m=0}^{\infty} \|\theta^{m+1}-\theta^m\|^2 <\infty.
        \end{align}
        \item There exists $C>0$ such that 
        \begin{align}
            \label{eq:T2}
            \text{dist}(0,\partial_{\rm mix} \Psi (\gamma^{m+1},\theta^{m+1})) \leq C \left( \|\gamma^{m+1} - \gamma^m\| + \| \theta^{m+1}-\theta^m \| \right).
        \end{align}

        \item The sequence $\{(\gamma^m,\theta^m)\}_{m\geq 0}$ has finite length:
        \begin{align}
            \label{eq:T3}
            \sum_{m=0}^{\infty} \left( \|\gamma^{m+1}-\gamma^{m}\| + \|\theta^{m+1}-\theta^m\| \right)<+\infty.
        \end{align}

        Consequently, the whole sequence converges to a mixed critical point $(\gamma^\star,\theta^\star)$ of $\Psi$.
    \end{itemize}
\end{theorem}

\begin{proof}
Write \(x^m=(\gamma^m,\theta^m)\) and \(\Delta\gamma^m=\gamma^{m+1}-\gamma^m\), \(\Delta\theta^m=\theta^{m+1}-\theta^m\).

\emph{Step 1: sufficient decrease.}
The variational characterization of the projected physical step gives
\[
\langle\nabla_\gamma J(\gamma^m,\theta^m),\Delta\gamma^m\rangle
\leq-\eta_\gamma^{-1}\|\Delta\gamma^m\|^2.
\]
The descent lemma from \textbf{(A3)} and \textbf{(A4)} therefore yields
\[
\Psi(\gamma^m,\theta^m)-\Psi(\gamma^{m+1},\theta^m)
\geq c_\gamma\|\Delta\gamma^m\|^2,\qquad
c_\gamma:=\eta_\gamma^{-1}-\ell_{\mathcal G}/2>0.
\]
For fixed \(\gamma^{m+1}\), \(J(\gamma^{m+1},\theta)\) differs from \(G^m(\theta)\) only by a constant. Hence \textbf{(A5)} gives
\[
\Psi(\gamma^{m+1},\theta^m)-\Psi(\gamma^{m+1},\theta^{m+1})
\geq a_1\|\Delta\theta^m\|^2.
\]
Adding the inequalities proves \eqref{eq:T1.0} with \(c=\min\{c_\gamma,a_1\}\). Since \(\Psi(x^m)\) decreases and is bounded below, telescoping proves \eqref{eq:T1.1}.

\emph{Step 2: relative error.}
Optimality of the projected step provides \(w^{m+1}\in N_{\mathcal G}(\gamma^{m+1})\) such that
\[
w^{m+1}=-\eta_\gamma^{-1}\Delta\gamma^m-
\nabla_\gamma J(\gamma^m,\theta^m).
\]
Set \(r_\gamma^{m+1}=\nabla_\gamma J(\gamma^{m+1},\theta^{m+1})+w^{m+1}\). By \textbf{(A3)},
\[
\|r_\gamma^{m+1}\|\leq
(\eta_\gamma^{-1}+\ell_{\mathcal G})\|\Delta\gamma^m\|
+\ell_{\mathcal G\Theta}\|\Delta\theta^m\|.
\]
By \textbf{(A5)}, choose \(r_\theta^{m+1}\in\partial^\circ G^m(\theta^{m+1})
=\partial_\theta^\circ J(\gamma^{m+1},\theta^{m+1})\) with
\(\|r_\theta^{m+1}\|\leq a_2\|\Delta\theta^m\|\). Thus
\((r_\gamma^{m+1},r_\theta^{m+1})\in\partial_{\rm mix}\Psi(x^{m+1})\), and \eqref{eq:T2} follows.

\emph{Step 3: critical accumulation points.}
All iterates lie in the compact set \(L_0\) by \textbf{(A2)} and lower semicontinuity. If \(x^{m_j}\to\bar x\), square summability implies \(x^{m_j+1}\to\bar x\), while Step 2 supplies \(r^{m_j+1}\in\partial_{\rm mix}\Psi(x^{m_j+1})\) with \(r^{m_j+1}\to0\). The graph of \(N_{\mathcal G}\) is closed because \(\mathcal G\) is closed and convex, and the graph of \((\gamma,\theta)\mapsto\partial_\theta^\circ J(\gamma,\theta)\) is closed on \(L_0\) by the outer semicontinuity in \textbf{(A1)} \cite{Clarke1990}. Passing to the limit gives \(0\in\partial_{\rm mix}\Psi(\bar x)\).

\emph{Step 4: finite length and whole-sequence convergence.}
Let \(E_m=\Psi(x^m)\), \(E_m\downarrow E_\infty\), and
\(q_m=\|\Delta\gamma^m\|+\|\Delta\theta^m\|\). Step 1 gives
\[
E_m-E_{m+1}\geq \frac c2q_m^2.
\]
The compact accumulation set \(\Omega\) is nonempty, consists of mixed critical points by Step 3, and \(\Psi\equiv E_\infty\) on \(\Omega\). The uniformized KL lemma \cite[Lemma~3.6]{BolteSabachTeboulle2014}, applied to the KL inequality assumed in \textbf{(A1)} with the same mapping \(\partial_{\rm mix}\Psi\), supplies a concave desingularizing function \(\varphi\) valid near \(\Omega\). Because \(\operatorname{dist}(x^m,\Omega)\to0\) and \(E_m\to E_\infty\), for all sufficiently large \(m\), either the sequence has already become stationary or
\[
\varphi'(E_m-E_\infty)
\operatorname{dist}(0,\partial_{\rm mix}\Psi(x^m))\geq1.
\]
Using Step 2 with the index shifted by one, there is \(b>0\) such that
$$\operatorname{dist}(0,\partial_{\rm mix}\Psi(x^m))\leq bq_{m-1}.$$

Concavity of \(\varphi\) and the decrease estimate then imply, for some \(K>0\),
\[
q_m^2\leq Kq_{m-1}\bigl[\varphi(E_m-E_\infty)-
\varphi(E_{m+1}-E_\infty)\bigr].
\]
The inequality \(2\sqrt{uv}\leq u+v\) yields
\[
q_m\leq\tfrac12q_{m-1}+\tfrac K2
\bigl[\varphi(E_m-E_\infty)-\varphi(E_{m+1}-E_\infty)\bigr].
\]
Summing and telescoping proves \(\sum_mq_m<\infty\), which is \eqref{eq:T3}. Thus \(x^m\) is Cauchy and converges in the closed set \(L_0\); Step 3 shows that its limit is mixed critical.
\end{proof}

\begin{remark}
The case \(\lambda_{\rm int}=0\) is included in
Theorem~\ref{thm:Alternating:Algorithm}. In this case, the alternating method reduces
to the two independent problems in \eqref{eq:separated:problems}.
\end{remark}

\begin{remark}[Scope and numerical solvers]
The theorem concerns the deterministic shared-observation core with fixed interaction points. Smooth \(\tanh\) and softplus parameterizations are compatible with the regularity assumptions, but the actual use of Adam or L-BFGS does not by itself guarantee the descent and relative-error conditions in \textbf{(A5)}. The result is an optimization-convergence statement: it neither identifies \(\gamma^\dagger\), implies global or Pareto optimality, nor covers the decentralized fragmented realization. It nevertheless supplies a rigorous descent-and-convergence foundation for the cooperative core.
\end{remark}

\section{Elliptic transmission inverse problem: interaction and Bi-Objective trade-offs}\label{sec:elliptic}

This experiment asks whether Bi-HYCO can recover parameters and states from spatially fragmented observations, whether state interaction contributes beyond coordinator aggregation, how scalarization changes the model-specific compromise, and how the reconstruction compares with PINN/XPINN references.

As a first numerical example, we consider a stationary inverse problem for an elliptic transmission equation on a domain with an irregular outer boundary and an internal interface. The diffusion coefficient is piecewise constant, taking different values on two subdomains. This experiment is designed to assess whether the Bi-HYCO strategy can recover the physical parameters and reconstruct the state when the available observations are spatially fragmented and when the interface geometry is non-trivial.

We use two scalar coefficients in order to isolate the effect of fragmented
observations and state interaction while keeping the parameter error and the
transmission diagnostics directly quantifiable. 

Complete implementation details for this experiment are provided in Appendix \ref{supp:elliptic-hyco-implementation}. The corresponding Python code is available on GitHub \cite{Github}. 

\subsection{PDE model and inverse problem} 

Let \(\Omega\) be a bounded domain in \(\mathbb{R}^2\) with an irregular outer boundary defined in polar coordinates by
\[
R_{\text{out}}(\theta) := 1.0 + 0.12\cos(3\theta) + 0.08\sin(5\theta),
\]
and an internal interface \(\Gamma\) given by
\[
R_{\text{int}}(\theta) := 0.42\bigl(1.0 + 0.22\cos(5\theta)\bigr).
\]

The interface splits \(\Omega\) into an inner subdomain \(\Omega_1 = \{\mathbf{x}: r < R_{\text{int}}(\theta)\}\) and an outer shell \(\Omega_2 = \Omega \setminus \overline{\Omega_1}\). We consider the elliptic boundary value problem

\begin{equation}\label{eq:elliptic-problem-irregular}
    \begin{array}{ll}
        \begin{cases}
            -\operatorname{div}\bigl(a(\mathbf{x})\nabla u(\mathbf{x})\bigr) = f(\mathbf{x}), & \mathbf{x}\in\Omega,\\
            u = 0, & \mathbf{x}\in\partial\Omega,
        \end{cases}
    & \quad a(\mathbf{x}) \coloneqq \begin{cases}
a_1, & \mathbf{x}\in\Omega_1,\\
a_2, & \mathbf{x}\in\Omega_2,
\end{cases} 
\\[15pt]  
& 
\quad f(\mathbf{x})\coloneqq 1 + 0.5\sin(\pi x)\cos(\pi y).
\end{array}
\end{equation}

The ground-truth parameters are $a_1^\dagger\coloneqq 2.0$ and $a_2^\dagger\coloneqq 1.0$. The inverse problem consists in recovering \(\gamma = (a_1,a_2)\in(0,\infty)^2\) from partial observations of \(u\). 

The observation set is partitioned into two local datasets. The first data-holding unit receives measurements in $\Omega_1$, whereas the second receives measurements in $\Omega_2$.

Since the diffusion coefficient \(a\) jumps across the internal interface \(\Gamma\), the boundary value problem \eqref{eq:elliptic-problem-irregular} is naturally interpreted as an elliptic transmission problem. Its weak formulation reads: find \(u \in H^1_0(\Omega)\) such that
\begin{align}
\label{variational:formulation:elliptic}
\int_{\Omega} a(\mathbf{x})\,\nabla u\cdot\nabla v\,\mathrm{d}\mathbf{x}
= \int_{\Omega} f(\mathbf{x})\,v\,\mathrm{d}\mathbf{x}, \qquad \forall\, v\in H^1_0(\Omega).
\end{align}

In this variational setting, the transmission conditions are not imposed explicitly. The continuity of the state, \([u]_\Gamma = 0\), is automatically ensured by the trace regularity of functions in \(H^1_0(\Omega)\). The continuity of the normal flux, \([a\,\partial_n u]_\Gamma = 0\), follows naturally from integration by parts and is therefore satisfied in the distributional sense. For a rigorous treatment of this variational framework, we refer to \cite{dautray1988mathematical}.

\subsection{Bi-HYCO implementation}

The physical component is obtained from a finite-element discretization of
\eqref{eq:elliptic-problem-irregular} on an unstructured Delaunay mesh, with
parameters listed in Table~\ref{tab:hyco_params} (see Appendix \ref{supp:elliptic-hyco-implementation});
the coefficient is assigned per triangle by its centroid, and homogeneous
Dirichlet conditions are imposed strongly. The assembly and solution are
implemented in PyTorch, rendering the physical model differentiable with
respect to \(a_1,a_2\). To avoid the inverse crime, the reference state
\(u^\dagger\) and observational data are generated on an independent finer mesh
(also reported in Table~\ref{tab:hyco_params}), while all reconstructions use
the coarser mesh described above. This two-mesh strategy prevents the
inversion from exploiting the same discretization used to produce the data.

Implementation details, including communication rounds, local updates, learning-rate schedules, architecture, and interaction-point sampling, are provided in Appendix \ref{supp:elliptic-hyco-implementation}.

\subsection{Sensitivity to observation noise}

To examine the sensitivity of Bi-HYCO to additive observation noise, we perturb the observations as follows.
\begin{equation}
    \widetilde y_i = y_i + \varepsilon_i, \quad \varepsilon_i \sim \mathcal{N}(0,\sigma^2),
\end{equation}
where the standard deviation is scaled relative to the clean data as $\sigma = \eta \cdot \operatorname{std}(\mathbf{y})$,
with \(\operatorname{std}(\mathbf{y})\) computed separately for the observations available to each local unit. The parameter \(\eta\) controls the relative noise level and is varied over \(\eta \in \{0.00, 0.05, 0.10, 0.15, 0.20\}\). 

\medskip
\noindent \textbf{Parameter recovery.}
Table~\ref{tab:elliptic-params} reports the estimated coefficients and their relative errors for the five noise levels. Here, the relative errors and the joint error are defined as
\begin{align*}
    \delta a_i \coloneqq \frac{|\hat a_i - a_i^\dagger|}{a_i^\dagger} \;\text{ for } i=1,2, \qquad e_{\gamma} := \sqrt{\delta a_1^2 + \delta a_2^2}.    
\end{align*}

\begin{table}[htbp]
\centering
\small
\begin{tabular}{lccccc}
\toprule
\multirow{2}{*}{Metric} & \multicolumn{5}{c}{\(\eta\)} \\
\cmidrule{2-6}
 & 0.00 & 0.05 & 0.10 & 0.15 & 0.20 \\
\midrule
$\hat a_1$ & 1.971312 & 1.998491 & 2.010475 & 2.001795 & 1.987880 \\
$\hat a_2$ & 1.001103 & 1.000469 & 1.001298 & 0.997622 & 0.999996 \\
$\delta a_1$ & 0.014344 & 0.000755 & 0.005237 & 0.000897 & 0.006060 \\
$\delta a_2$ & 0.001103 & 0.000469 & 0.001298 & 0.002378 & 0.000004 \\
$e_{\gamma}$ & 0.014386 & 0.000888 & 0.005396 & 0.002541 & 0.006060 \\
Run time (s) & 2161.63 & 2298.08 & 2407.73 & 2265.47 & 2181.89 \\
\bottomrule
\end{tabular}
\caption{Estimated diffusion coefficients and relative errors under increasing observation noise.}
\label{tab:elliptic-params}
\end{table}

In all cases, the joint parameter error remains below $1.5\%$, with the smallest value of $0.089\%$ achieved at $\eta=0.05$. The relative error in \(a_2\) remains below $2.4\times 10^{-3}$ for all five perturbation levels, while the error in \(a_1\) varies between $0.08\%$ and $1.43\%$. No monotone deterioration with respect to $\eta$ is observed in these runs. Since each noise level corresponds to a fixed realization, the nonmonotone behavior should be interpreted descriptively rather than as evidence that observational noise improves the reconstruction.

The five noise levels are distinct experimental conditions rather than statistical replicates; no confidence or uncertainty band is inferred from their dispersion.

\medskip 
\noindent \textbf{State reconstruction across subdomains.}
Table~\ref{tab:elliptic-solution-errors} shows that the two local units attain different levels of state accuracy. The relative error of a reconstructed
solution $u$ with respect to the true solution $u^\dagger$ on a subdomain $\omega$ is 
\begin{align*}
    e_u(\omega) := \frac{\|u-u^\dagger\|_{L^2(\omega)}}{\|u^\dagger\|_{L^2(\omega)}}.
\end{align*}

\begin{table}[htbp]
\centering
\small
\begin{tabular}{lccccc}
\toprule
\multirow{2}{*}{Metric} & \multicolumn{5}{c}{\(\eta\)} \\
\cmidrule{2-6}
 & 0.00 & 0.05 & 0.10 & 0.15 & 0.20 \\
\midrule
$e_{\rm phy}(\Omega)$ & 0.0030 & 0.0030 & 0.0030 & 0.0044 & 0.0032 \\
$e_{\rm syn}(\Omega)$ & 0.0821 & 0.0648 & 0.0614 & 0.0400 & 0.0664 \\
$e_{\rm phy}(\Omega_1)$ & 0.0013 & 0.0010 & 0.0018 & 0.0022 & 0.0010 \\
$e_{\rm phy}(\Omega_2)$ & 0.0039 & 0.0040 & 0.0038 & 0.0057 & 0.0042 \\
$e_{\rm syn}(\Omega_1)$ & 0.0266 & 0.0187 & 0.0150 & 0.0163 & 0.0156 \\
$e_{\rm syn}(\Omega_2)$ & 0.1102 & 0.0874 & 0.0833 & 0.0528 & 0.0901 \\
$\mathrm{Gap}_{\rm phy}$ & 0.0026 & 0.0030 & 0.0020 & 0.0035 & 0.0032 \\
$\mathrm{Gap}_{\rm syn}$ & -0.0836 & -0.0687 & -0.0683 & -0.0365 & -0.0745 \\
\bottomrule
\end{tabular}
\caption{Relative $L^2$ reconstruction errors and cross-subdomain error differences across noise levels.}
\label{tab:elliptic-solution-errors}
\end{table}
The solver-based reconstruction has a global relative error of order
\(10^{-3}\), whereas the synthetic reconstruction has an error of order
\(10^{-2}\). This difference is consistent with the fact that the former is
obtained by solving the discretized transmission problem for the recovered
coefficients, while the latter is an unconstrained state representation.

The quantities $\mathrm{Gap}_{\rm phy}:=
e_{\rm phy}(\Omega_2)-e_{\rm phy}(\Omega_1)$ and
$\mathrm{Gap}_{\rm syn}:=
e_{\rm syn}(\Omega_1)-e_{\rm syn}(\Omega_2)$ are differences between reconstruction errors on
the two subdomains. They are used only as spatial diagnostics and should not
be interpreted as statistical out-of-sample error estimates. The physical difference
is close to zero in all cases, so the finite-element reconstruction has
comparable accuracy in \(\Omega_1\) and \(\Omega_2\). The synthetic difference
is negative in the reported runs, indicating a smaller error in
\(\Omega_1\) than in \(\Omega_2\). This observation is compatible with
information exchange through the coupled procedure.

\medskip

\noindent \textbf{Sensitivity with respect to the scalarization weights.}
We also investigate the sensitivity of the reconstruction to the
scalarization weights in \eqref{eq:generic-scalarization}. To this end, we fix the noise level at
$\eta=0.20$ and consider
\[
\omega_{\rm phy}\in\{1.0,0.75,0.50,0.25,0.0\},
\qquad
\omega_{\rm syn}:=1-\omega_{\rm phy},
\]
while keeping all remaining algorithmic parameters fixed.

The weight $\omega_{\rm phy}$ determines the relative importance of the physical objective $F_{\rm phy}$ in the scalarized functional \eqref{eq:generic-scalarization}, while $\omega_{\rm syn}$ weights the synthetic objective $F_{\rm syn}$. 

\begin{table}[htbp]
\centering
\small
\begin{tabular}{lccccc}
\toprule
$(\omega_{\mathrm{phy}}, \omega_{\mathrm{syn}})$ & $\hat a_1$ & $\hat a_2$ & $\delta a_1$ & $\delta a_2$ & $e_\gamma$ \\
\midrule
(1.00, 0.00) & 1.964579 & 1.000530 & 0.017711 & 0.000530 & 0.017719 \\
(0.75, 0.25) & 2.032572 & 0.996496 & 0.016286 & 0.003504 & 0.016659 \\
(0.50, 0.50) & 1.987880 & 0.999996 & 0.006060 & 0.000004 &  0.006060 \\
(0.25, 0.75) & 2.211544 & 0.992827 & 0.105772 & 0.007173 & 0.106015 \\
(0.00, 1.00) & 1.260271 & 0.993013 & 0.369864 & 0.006987 & 0.369930 \\
\bottomrule
\end{tabular}
\caption{Sensitivity of the recovered coefficients to the scalarization
weights $(\omega_{\rm phy},\omega_{\rm syn})$ when $\eta=0.20$.
All remaining hyperparameters are kept fixed.}
\label{tab:weight-sensitivity}
\end{table}

The results in Table \ref{tab:weight-sensitivity} show that the reconstruction is sensitive to the relative weighting of the two criteria. Small parameter errors are obtained when the physical criterion carries a sufficiently large weight, while the reconstruction deteriorates as the scalarization becomes increasingly dominated by the synthetic criterion.

\medskip
\noindent\textbf{Vanishing-interaction ablation.}
To isolate the contribution of explicit state-level interaction, we also repeated the
elliptic experiment at \(\eta=0.20\) with \(\lambda_{\rm int}=0\), while leaving the
fragmented data distribution, local optimization procedure, and coordinator aggregation
unchanged. The joint parameter error increases from \(e_\gamma=0.00606\) for full Bi-HYCO
to \(e_\gamma=0.00896\) in the aggregation-only configuration, an increase by a factor of approximately \(1.48\) in this realization. The complete ablation results and discussion are reported in Appendix \ref{supp:elliptic-vanishing-interaction}.

\subsection{Comparison with Physics-Informed Neural Network Variants}

For comparison, we also compute standard PINN and XPINN reconstructions in the same noise-free setting. 

Figure~\ref{fig:solution_comparison} compares the reconstructed states and their absolute errors. In this noise-free realization, both Bi-HYCO states reproduce the main spatial structure of the reference, with the solver-based state attaining the smallest global error in Table~\ref{tab:hyco_pinn_compare}. The PINN and XPINN states show larger deviations in parts of the inner subdomain and near the interface, as highlighted by the zoomed views.

\begin{figure}[htbp]
    \centering
    \includegraphics[width=0.95\textwidth]{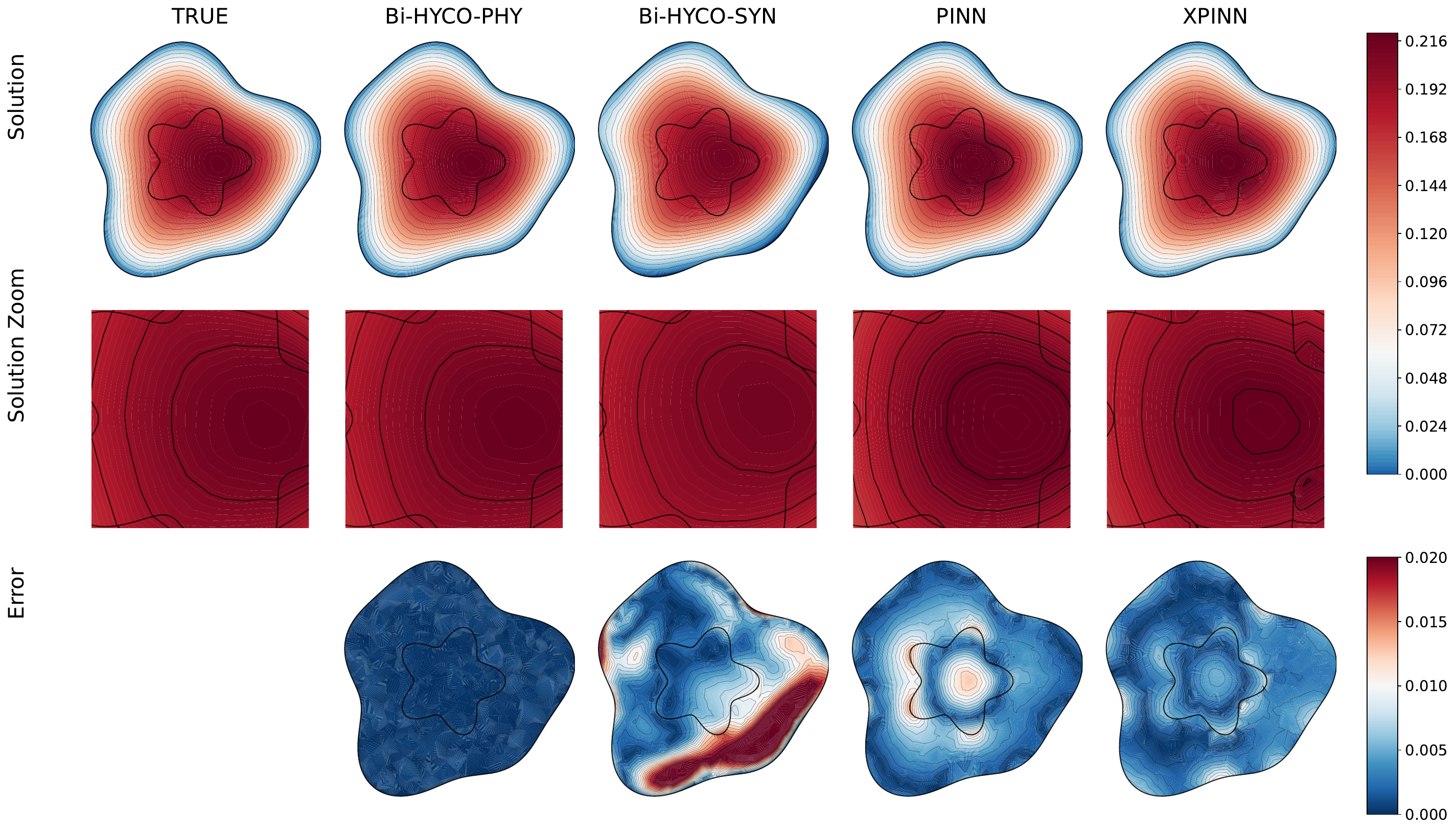}
    \caption{Solution reconstruction comparison at \(\eta=0.00\). Rows from top to bottom: global solution, zoomed solution (interface region), and absolute error. Columns from left to right: true solution, Bi-HYCO physics, Bi-HYCO synthetic, PINN, and XPINN. The bottom row uses a linear color scale with \(v_{\min}=0\) and \(v_{\max}\) set to the 98th percentile of all errors.}
    \label{fig:solution_comparison}
\end{figure}

Figure~\ref{fig:gradient_comparison} shows the corresponding gradient magnitudes. The solver-based Bi-HYCO reconstruction is visually closer to the reference in several regions, while PINN and XPINN exhibit larger local discrepancies near the interface. Since gradients are particularly sensitive to discretization and post-processing, these plots are interpreted only as qualitative diagnostics.

\begin{figure}[htbp]
    \centering
    \includegraphics[width=0.95\textwidth]{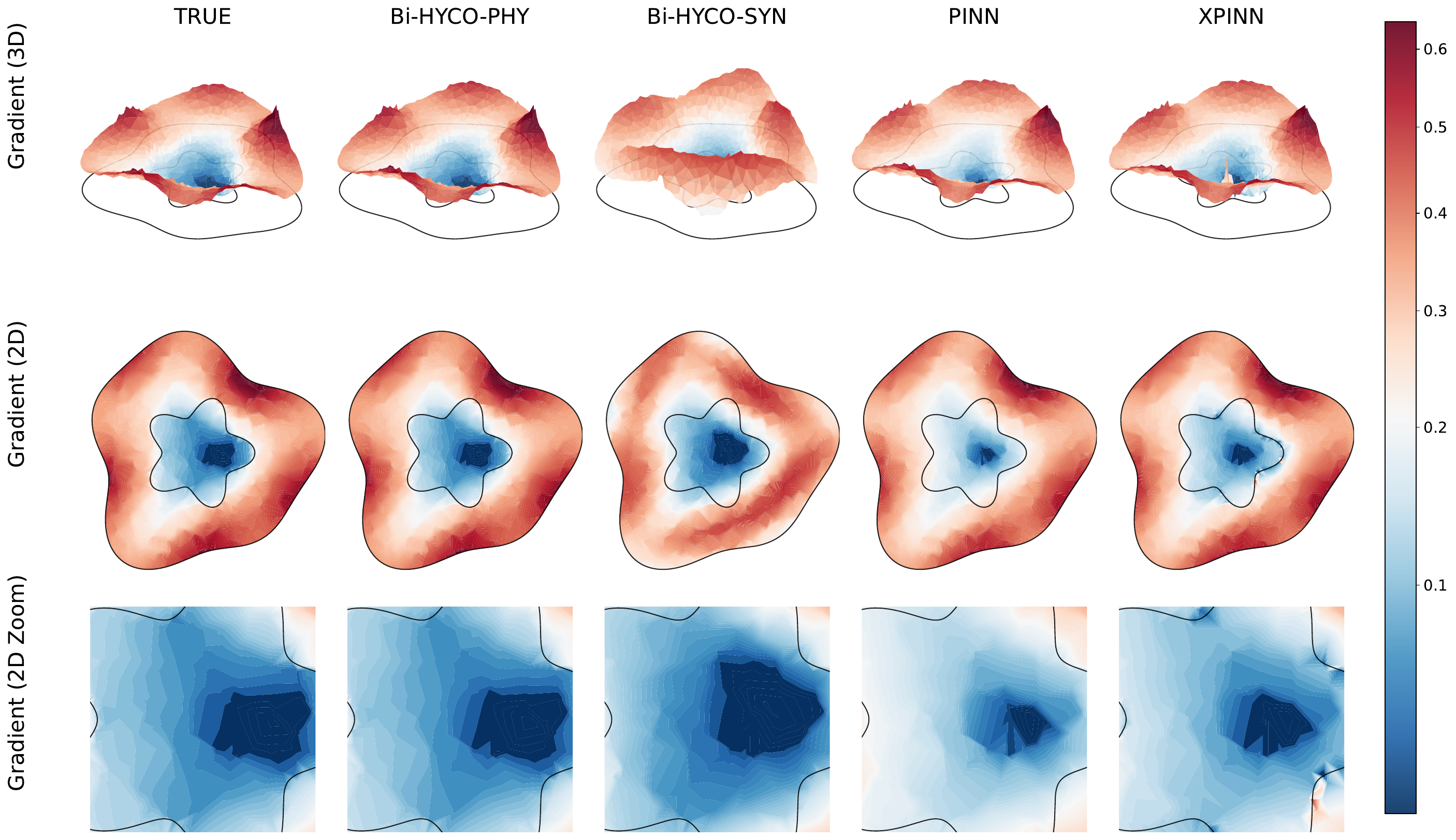}
    \caption{Gradient magnitude comparison at \(\eta=0.00\). Rows from top to bottom: three-dimensional perspective of the global gradient field, two-dimensional global view, and two-dimensional zoomed view near the interface. Columns from left to right: true solution, Bi-HYCO physics, Bi-HYCO synthetic, PINN, and XPINN. The first row employs a \(3D\) surface plot to highlight the spatial distribution of gradient magnitudes.}
    \label{fig:gradient_comparison}
\end{figure}

Table~\ref{tab:hyco_pinn_compare} summarizes the quantitative comparison. In this noise-free realization, Bi-HYCO achieves the lowest parameter-recovery error and its solver-based state yields the lowest global state error. The particularly small error of the physical Bi-HYCO state should, however, be interpreted in light of the fact that it is generated by the finite-element solver after coefficient identification. Therefore, the comparison does not isolate the effect of the optimization strategy from that of solver-based state reconstruction.

% Table 5: Comparison between HYCO, PINN and XPINN at η=0.00
\begin{table}[htbp]
\centering
\small
\setlength{\tabcolsep}{10pt}
\begin{tabular}{lccc}
\toprule
\multirow{2}{*}{Metric} & \multicolumn{3}{c}{Method} \\
\cmidrule{2-4}
                        & Bi-HYCO & PINN & XPINN \\
\midrule
$\hat a_1$ & 1.9713 & 1.0788 & 1.3579 \\
$\hat a_2$ & 1.0011 & 1.0481 & 1.0355 \\
$\delta a_1$ & 0.0143 & 0.4606 & 0.3210 \\
$\delta a_2$ & 0.0011 & 0.0481 & 0.0355 \\
$e_\gamma$ & 0.0144 & 0.4631 & 0.3230 \\
\midrule
$e_{\mathrm{phy}}(\Omega)$ & 0.0030 & \multirow{2}{*}{0.0378} & \multirow{2}{*}{0.0248} \\
$e_{\mathrm{syn}}(\Omega)$ & 0.0821 & & \\
$e_{\mathrm{phy}}(\Omega_1)$ & 0.0013 & \multirow{2}{*}{0.0370} & \multirow{2}{*}{0.0179} \\
$e_{\mathrm{syn}}(\Omega_1)$ & 0.0266 & & \\
$e_{\mathrm{phy}}(\Omega_2)$ & 0.0039 & \multirow{2}{*}{0.0385} & \multirow{2}{*}{0.0296} \\
$e_{\mathrm{syn}}(\Omega_2)$ & 0.1102 & & \\
$\mathrm{Gap}_{\mathrm{phy}}$ & 0.0026 & \multirow{2}{*}{0.0015} & \multirow{2}{*}{0.0116} \\
$\mathrm{Gap}_{\mathrm{syn}}$ & -0.0836 & & \\
\midrule
$[u]_{\mathrm{phy}}$ & 0.00004 & \multirow{2}{*}{0.00003} & \multirow{2}{*}{0.00005} \\
$[u]_{\mathrm{syn}}$ & 0.00003 & & \\
$[a\nabla u\cdot n]_{\mathrm{phy}}$ & 0.00174 & \multirow{2}{*}{0.0054} & \multirow{2}{*}{0.0744} \\
$[a\nabla u\cdot n]_{\mathrm{syn}}$ & 0.00143 & & \\
\midrule
Run time (s) & 2161.6 & 882.1 & 1019.8 \\
\bottomrule
\end{tabular}
\caption{Comparison of the reconstruction results obtained with Bi-HYCO, PINN and XPINN at noise level $\eta=0.00$}
\label{tab:hyco_pinn_compare}
\end{table}

Finally, Bi-HYCO requires approximately \(2.12\)-\(2.45\) times the wall-clock time of the PINN and XPINN implementations. This additional computational cost is accompanied by lower parameter-recovery and physical-state errors, so the relevance of the trade-off depends on whether computational cost or parameter accuracy is the primary criterion.

\section{Navier-Stokes Parameter Identification: A Nonlinear Fragmented-Observation Test}
\label{sec:annular-ns}

We next consider a time-dependent inverse problem in the annular domain with
inner and outer radii \(r_{\text{in}}=0.25\) and \(r_{\text{out}}=1.0\). The
unknown physical parameters are recovered from spatially fragmented
observations of the vorticity field. This example examines the same
reconstruction mechanism for a nonlinear evolution equation and a curved
geometry.

Complete implementation details are provided in Appendices \ref{supp:ns-hyco-implementation}-\ref{supp:ns-initialization}-\ref{supp:ns-pinn-implementation}. The
Python code is available on GitHub \cite{Github}.

\subsection{PDE model, observations and Bi-HYCO setup}

Let 
\begin{align*}
    \Omega = \{(r,\theta): r_{\text{in}} < r < r_{\text{out}},\ 0\le\theta<2\pi\}
\end{align*} 
be the annular domain, and \(T>0\). We employ the vorticity-streamfunction formulation in polar coordinates:
\begin{equation}
\label{eq:annular-ns}
\begin{cases}
    \partial_t\omega + \dfrac{1}{r}\partial_r(r u_r \omega) + \dfrac{1}{r}\partial_\theta(u_\theta \omega)
        - \nu \Delta \omega = \rho\, f(r,\theta), & (r,\theta,t)\in\Omega\times(0,T], \\
    -\Delta \psi = \omega, & (r,\theta,t)\in\Omega\times(0,T], \\
    u_r = \dfrac{1}{r}\partial_\theta\psi,\quad u_\theta = -\partial_r\psi, & (r,\theta,t)\in\Omega\times(0,T], \\
    \omega(r,\theta,0) = \omega_0(r,\theta), & (r,\theta)\in\Omega, \\
    \psi(r_{\text{in}},\theta,t) = \psi(r_{\text{out}},\theta,t) = 0, & \theta\in[0,2\pi), \\
    \omega(r_{\text{in}},\theta,t) = \omega(r_{\text{out}},\theta,t) = 0, & \theta\in[0,2\pi), \\
    \psi(r,0,t) = \psi(r,2\pi,t),\quad \omega(r,0,t) = \omega(r,2\pi,t), & r\in[r_{\text{in}},r_{\text{out}}].
\end{cases}
\end{equation}

Here \(\nu>0\) denotes the kinematic viscosity and \(\rho\) the amplitude of a prescribed external forcing. The condition \(\psi=0\) fixes the streamfunction on each circular boundary and enforces zero normal velocity. The condition \(\omega=0\) prescribes homogeneous wall vorticity. On the curved boundaries of the annulus, it is regarded here as the mathematical wall condition defining the numerical model, rather than as the general stress-free condition, which involves curvature-dependent terms. Angular periodicity is enforced by the last line of \eqref{eq:annular-ns}. The forcing term is designed to produce a characteristic two-ring vortex structure and is given by
\[
f(r,\theta) := \exp\!\left(-\frac{(r-0.62)^2}{0.13^2}\right)\sin(3\theta+0.25)
            + 0.35\,\exp\!\left(-\frac{(r-0.82)^2}{0.10^2}\right)\cos(5\theta-0.4).
\]
The initial vorticity field is constructed as a superposition of four Gaussian vortices,
\[
\omega_0(r,\theta) := \sum_{k=1}^4 A_k \exp\!\left(-\frac{(r-r_k)^2 + r_k^2\, d_S(\theta,\theta_k)^2}{2\sigma^2}\right),
\]
where
\[
d_S(\theta,\theta_k) := \operatorname{arctan}\bigl(\sin(\theta-\theta_k), \cos(\theta-\theta_k)\bigr)
\]
denotes the periodic angular distance on the unit circle, and \(\sigma = 0.085\).
The parameters \((r_k,\theta_k,A_k)\) are listed in Table~\ref{tab:annular_ic_params}.

\begin{table}[htbp]
\centering
\small
\begin{tabular}{c|cccc}
\toprule
$k$        & 1    & 2    & 3    & 4    \\
\midrule
$r_k$      & 0.55 & 0.70 & 0.58 & 0.78 \\
$\theta_k$ & 0.20 & 1.70 & 3.35 & 5.00 \\
$A_k$      & 2.6  & -2.2 & 2.4  & -2.0 \\
\bottomrule
\end{tabular}
\caption{Parameters of the four Gaussian vortices composing the initial condition for the annular domain experiments.}
\label{tab:annular_ic_params}
\end{table}

The inverse problem consists in identifying $\gamma := (\nu,\rho) \in (0,\infty)^2$. To this end, a high-fidelity reference solution \(\omega^\dagger\) is generated with ground-truth parameters $\nu^\dagger := 10^{-3}$ and $\rho^\dagger := 1$ using a second-order finite-difference scheme on a \(128\times256\) polar grid with time step \(dt_{\text{fine}}=3.5\times10^{-3}\). 

The observation set $\mathcal D := \{(r_i,\theta_i,t_i,\omega_i^\dagger)\}_{i=1}^{M}$, with $\omega_i^\dagger:=\omega^\dagger(r_i,\theta_i,t_i)$,
is split into four angular sectors. The measurements are split into two local datasets: $D_1$ containing observations from sectors $0$ and $2$ and $D_2$ with observations from sectors $1$ and $3$.

The physical model solves \eqref{eq:annular-ns} on a coarser \(64\times128\) grid with time step \(dt=5\times10^{-3}\) using the same second-order scheme; its prediction is denoted \(\omega_{\text{phy}}(\cdot,\cdot,\cdot;\gamma)\). The synthetic model is a fully connected neural network with Fourier feature embedding; its architecture and optimization details are given in Appendices \ref{supp:ns-hyco-implementation}-\ref{supp:ns-initialization}. The physical parameters are optimized with L-BFGS-B, while the synthetic network parameters are optimized with Adam. The losses \(\mathcal L_{\text{phy}}\), \(\mathcal L_{\text{syn}}\) and \(\mathcal L_{\text{int}}\) are defined as in the general framework, with weights specified in Appendices \ref{supp:ns-hyco-implementation}-\ref{supp:ns-initialization}. The coordinator aggregates the local parameter pairs by simple arithmetic averaging. 

\subsection{Sensitivity to observation noise}
We now examine the sensitivity of the algorithm to additive observation noise. The relative noise level
\[
\eta\in\{0.00,\,0.05,\,0.10,\,0.15,\,0.20\}
\]
is applied independently to the physical and synthetic observation sets. Table~\ref{tab:hyco-noise-param} summarizes the recovered parameters, relative errors, interaction-point consistency, and runtime.

\begin{table}[htbp]
\centering
\small
\begin{tabular}{@{}lccccc@{}}
        \toprule
        \multirow{2}{*}{Metric} & \multicolumn{5}{c}{\(\eta\)} \\
        \cmidrule(lr){2-6}
        & 0.00 & 0.05 & 0.10 & 0.15 & 0.20 \\
        \midrule
        $\hat\nu$ & 1.020e-03 & 1.191e-03 & 1.093e-03 & 1.065e-03 & 1.155e-03 \\
        $\hat\rho$ & 1.00031 & 1.00986 & 1.00463 & 0.99771 & 0.99864 \\
        $\delta\nu$ & 0.01977 & 0.19066 & 0.09254 & 0.06486 & 0.15532 \\
        $\delta\rho$ & 0.00031 & 0.00986 & 0.00463 & 0.00229 & 0.00136 \\
        $e_\gamma$ & 0.01978 & 0.19091 & 0.09265 & 0.06490 & 0.15532 \\
        $e_{\rm int}^{\rm train}$ & 1.024e-04 & 2.013e-04 & 3.709e-04 & 7.830e-04 & 1.337e-03 \\
        $e_{\rm int}^{\rm test}$ & 1.121e-04 & 1.982e-04 & 3.886e-04 & 7.816e-04 & 1.301e-03 \\
        Runtime (s) & 3330.19 & 3631.76 & 3797.15 & 3372.38 & 3885.40 \\
        \bottomrule
    \end{tabular}
    \caption{Parameter recovery and interaction-point agreement under increasing relative noise $\eta$.}
    \label{tab:hyco-noise-param}
\end{table}

All methods use the effective feasible initialization \((\nu_0,\rho_0)=(5\times10^{-3},0.5)\). The projection of the nominal unconstrained value is documented in Appendix \ref{supp:ns-initialization}. As in the elliptic experiment, the five noise levels are different conditions, not statistical replicates.

\medskip
\noindent\textbf{Vanishing-interaction ablation.}
We also repeated the annular Navier-Stokes experiment at \(\eta=0.20\) with
\(\lambda_{\rm int}=0\), keeping the fragmented observations, local optimization, and
coordinator aggregation unchanged. In this matched comparison, the joint parameter
error increases from \(e_\gamma=0.15532\) for full Bi-HYCO to \(e_\gamma=0.92131\) when the
explicit interaction term is removed. The full results are given in Appendix
\ref{supp:ns-vanishing-interaction}.
Because fragmentation, local optimization, and aggregation are unchanged, this matched run isolates the contribution of \(\mathcal L_{\rm int}\) relative to aggregation-only cooperation; it does not establish a general robustness or identifiability result.

\subsection{Comparison with Physics-Informed Neural Network variants}

The following comparison is carried out in the same spirit as in the elliptic example. The PINN and XPINN calculations are used as reference reconstructions under the same observation geometry and are not intended as a systematic comparison with classical PDE-constrained optimization. In particular, a classical centralized method supplied with the union of the fragmented datasets would solve a different information problem; see Section~\ref{sec:conclusion}.

Table~\ref{tab:annular-hyco-pinn-compare} summarizes the parameter and state reconstruction results. XPINN achieves the smallest joint parameter error in this realization, mainly because of its more accurate viscosity estimate, whereas Bi-HYCO provides the most accurate forcing-amplitude estimate and the smallest global state-reconstruction error.

\begin{table}[htbp]
\centering
\small
\setlength{\tabcolsep}{10pt}
\begin{tabular}{@{\hspace{2mm}} l ccc @{\hspace{2mm}}}
\toprule
\multirow{2}{*}{Metric} & \multicolumn{3}{c}{Method} \\
\cmidrule(lr){2-4}
                        & Bi-HYCO & PINN & XPINN \\
\midrule
$\hat\nu$ & 0.0010198 & 0.0009393 & 0.0009877 \\
$\hat\rho$ & 1.0003081 & 0.9894614 & 0.9922372 \\
$\delta\nu$ & 0.0197740 & 0.0606880 & 0.0122906 \\
$\delta\rho$ & 0.0003081 & 0.0105386 & 0.0077628 \\
$e_\gamma$ & 0.0197764 & 0.0615962 & 0.0145369 \\
$e_{\rm phy}(\Omega)$ & 0.0021712 & \multirow{2}{*}{0.0205557} & \multirow{2}{*}{0.0163554} \\
$e_{\rm syn}(\Omega)$ & 0.0221767 & \multicolumn{1}{c}{} & \multicolumn{1}{c}{} \\
$e_{\rm phy}(\Omega_1)$ & 0.0022837 & \multirow{2}{*}{0.0180420} & \multirow{2}{*}{0.0143055} \\
$e_{\rm syn}(\Omega_1)$ & 0.0256010 & \multicolumn{1}{c}{} & \multicolumn{1}{c}{} \\
$e_{\rm phy}(\Omega_2)$ & 0.0019986 & \multirow{2}{*}{0.0237020} & \multirow{2}{*}{0.0189131} \\
$e_{\rm syn}(\Omega_2)$ & 0.0160312 & \multicolumn{1}{c}{} & \multicolumn{1}{c}{} \\
$\mathrm{Gap}_{\rm phy}$ & -0.0002852 & \multirow{2}{*}{0.0056600} & \multirow{2}{*}{0.0046077} \\
$\mathrm{Gap}_{\rm syn}$ & 0.0095699 & \multicolumn{1}{c}{} & \multicolumn{1}{c}{} \\
Runtime (s) & 3330.1 & 1263.2& 5147.3 \\
\bottomrule
\end{tabular}
\caption{Comparison of Bi-HYCO, PINN, and XPINN on the Navier-Stokes inverse problem (noise level $\eta=0.00$).}
\label{tab:annular-hyco-pinn-compare}
\end{table}

For state reconstruction, the solver-based Bi-HYCO component is the most accurate of the considered approaches and exhibits comparable errors across the two observation sectors. The synthetic component also retains nontrivial accuracy in the sector without direct observations. This behavior is compatible with information exchange in the coupled procedure, although the comparison by itself does not separate the contribution of state interaction from that of parameter aggregation.

The measured runtime places Bi-HYCO between PINN and the four-network XPINN implementation. These timings are implementation-dependent and should be interpreted only within the computational setting described in Appendices \ref{supp:ns-hyco-implementation}-\ref{supp:ns-initialization}.

The solver-based Bi-HYCO errors are generally smaller than those of PINN and XPINN, consistently with the global errors reported in Table~\ref{tab:annular-hyco-pinn-compare}. The XPINN reconstruction also displays oscillatory patterns near some angular-sector interfaces.

Finally, Figure~\ref{fig:gradient} compares the gradient magnitude \(|\nabla\omega|\). In the reported realization, the principal gradient peaks of the solver-based Bi-HYCO reconstruction are closer in location and amplitude to those of the reference field. As in the elliptic example, this comparison is intended as a qualitative derivative diagnostic rather than as a separate consistency result.

\begin{figure}[t]
  \centering
  \includegraphics[width=0.90\textwidth]{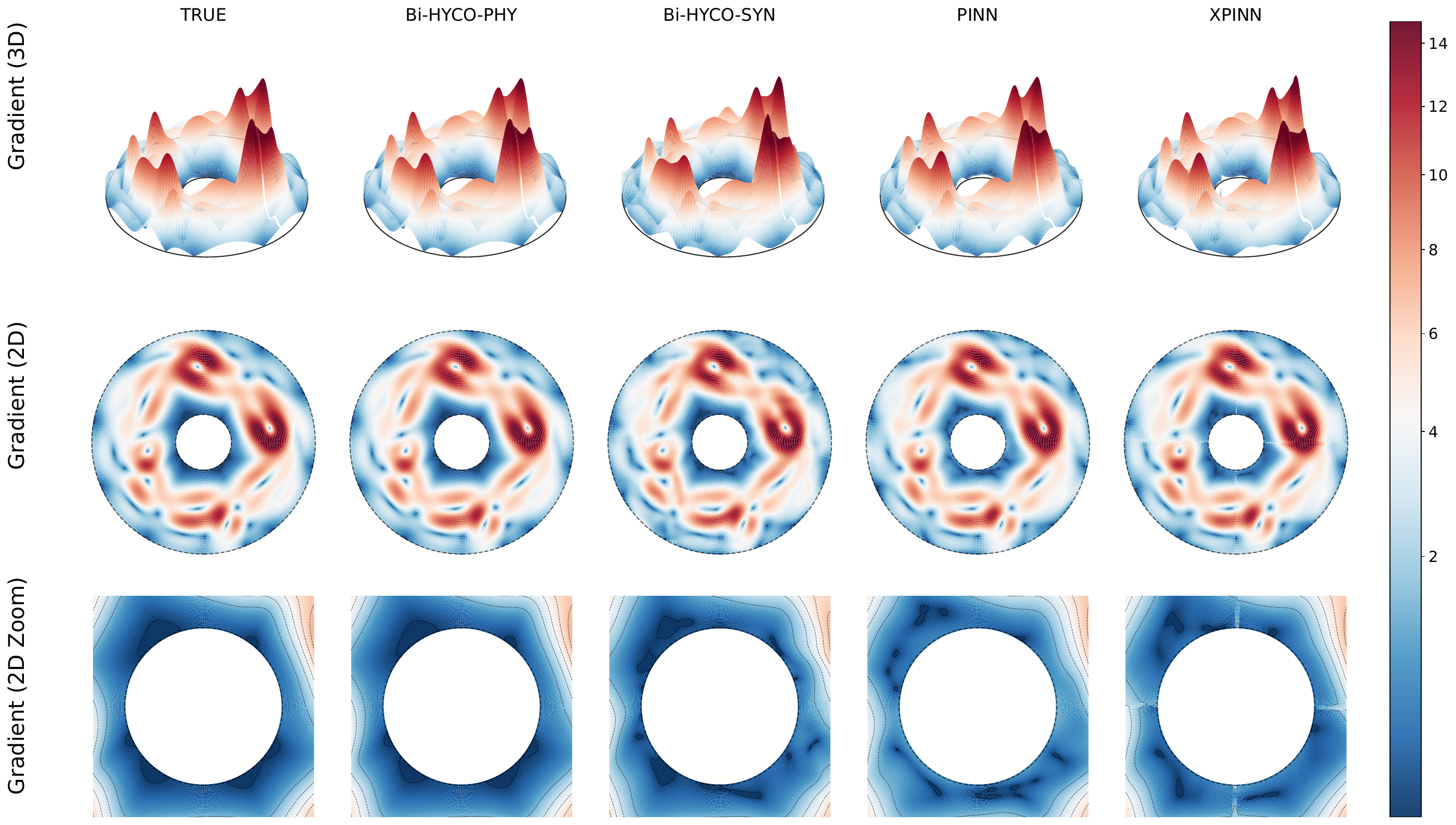}
  \caption{Gradient magnitude \(|\nabla \omega| = \sqrt{(\partial_r \omega)^2 + (r^{-1}\partial_\theta \omega)^2}\) at \(t=1.5\). The first row presents a 3D surface view, the second row a 2D top-down view, and the third row a zoomed 2D view of the same central region. }
  \label{fig:gradient}
\end{figure}

\section{Conclusions, discussion and perspectives}
\label{sec:conclusion}

Bi-HYCO provides a cooperative framework for combining physical and synthetic representations when the available observations are fragmented. The two models retain their own observation-based objectives and interact by comparing their predicted states at unlabeled points. This leads naturally to a bi-objective formulation, with weighted scalarization providing a practical way to balance the two components. For the alternating scheme with shared observations, we establish convergence of the whole sequence to a mixed critical point. In the fragmented-observation setting, the numerical experiments show that state interaction can improve parameter recovery beyond the effect of parameter aggregation alone.

The framework is intended for problems in which different observations are available to different models, for instance across spatial regions, time intervals, sensors, or computational units. In this setting, the models can exchange information through their predicted states without directly sharing their observation sets. This is particularly useful when the physical and synthetic representations provide different but complementary descriptions of the same system.

More generally, the results indicate that agreement between predicted states can provide a simple way for heterogeneous models to use information learned from different observations while preserving their distinct representations.

Several questions remain for future work. From the analytical viewpoint, it would be interesting to extend the convergence analysis to the decentralized algorithm, to investigate how state interaction affects parameter identifiability, and to obtain stability estimates involving \(\mathcal L_{\rm int}\). From the computational viewpoint, natural extensions include adaptive selection of the interaction points, higher-dimensional parameter spaces, and configurations involving several models or interfaces.

\section*{Acknowledgments}

The authors thank Lorenzo Liverani and Ziqian Li for fruitful discussions
and for valuable suggestions concerning the numerical implementation and
the visualization of the results.

\paragraph{Code and data availability}
The source code, experiment configurations, and data-generation procedures used for the reported computations are available in the accompanying repository \cite{Github}; the Supplementary Materials record the principal settings needed for reproduction.

{\appendix
	
\section{Experimental implementation details}\label{appendix}

This appendix provides implementation, reproducibility, ablation, and diagnostic material for the Bi-HYCO experiments in the main article. The complete proof of Theorem~\ref{thm:Alternating:Algorithm} is contained in the refereed main manuscript.

Section~\ref{supp:computing-environment} records the computing environment. Section~\ref{supp:elliptic-additional} contains the elliptic interaction ablation and complete Bi-HYCO/PINN/XPINN settings. Section~\ref{supp:ns-additional} gives the corresponding Navier--Stokes ablation, initialization, and implementation details. These items support reproduction and interpretation of the numerical results but are not required for the mathematical validity of the main convergence theorem.

\subsection{Computing environment and reproducibility}
\label{supp:computing-environment}

All simulations were conducted on a laptop running the 64-bit edition of Windows~11 (build 26100). The system was equipped with an Intel Core Ultra~5 225H processor (14 cores, 1.7~GHz base frequency, x86-64/AMD64 architecture) and 32~GB of RAM. The algorithms were implemented in Python~3.13.5 using the Anaconda distribution.

\subsection{Elliptic transmission experiment}
\label{supp:elliptic-additional}

\subsubsection{Vanishing-interaction experiment}
\label{supp:elliptic-vanishing-interaction}\label{subsec:elliptic-vanishing-interaction}

We examine parameter recovery in the vanishing-interaction regime
\(\lambda_{\rm int}=0\) at the highest noise level considered in the
experiment, namely \(\eta=0.20\). This case is used as an internal reference
for the role of the explicit interaction term. The data distribution, local
optimization procedure, and coordinator aggregation remain unchanged, while
the physical and synthetic states are no longer directly coupled through
\(\mathcal{L}_{\rm int}\). The resulting estimates are compared with the
Bi-HYCO values reported in Table~\ref{tab:elliptic-params} for the same
noise level.

\begin{table}[htbp]
	\centering
	\small
	\begin{tabular}{lcc}
		\toprule
		Metric
		& Full Bi-HYCO
		& Vanishing interaction, $\lambda_{\rm int}=0$
		\\
		\midrule
		$\widehat a_1$
		& $1.98788$
		& $1.98213$
		\\
		$\widehat a_2$
		& $0.99999$
		& $1.00066$
		\\
		$\delta_{a_1}$
		& $0.00606$
		& $0.00893$
		\\
		$\delta_{a_2}$
		& $4.00\times 10^{-6}$
		& $0.00066$
		\\
		$e_\gamma$
		& $0.00606$
		& $0.00896$
		\\
		\bottomrule
	\end{tabular}
	\caption{Parameter recovery for the elliptic transmission inverse problem
		at $\eta=0.20$. The Bi-HYCO values are those reported in
		Table~\ref{tab:elliptic-params}.}
	\label{tab:elliptic-vanishing-interaction}
\end{table}

Relative to the Bi-HYCO values reported in Table~\ref{tab:elliptic-params}, the vanishing-interaction configuration produces a larger joint parameter error. More precisely, \(e_\gamma\) increases from \(0.006060\) to \(0.00896\), corresponding to a factor of approximately \(1.48\). The relative error in \(a_1\) increases from \(0.006060\) to \(0.00893\), while the error in \(a_2\) increases from \(4\times10^{-6}\) to \(6.6\times10^{-4}\).

Thus, in this realization, removing the explicit interaction term deteriorates the recovery of both coefficients, although the effect is substantially more pronounced for \(a_1\) in absolute terms. According to the joint error \(e_\gamma\), the Bi-HYCO configuration therefore yields the more accurate parameter reconstruction at \(\eta=0.20\). This calculation suggests that the interaction term contributes favorably to parameter recovery in the prescribed noisy configuration, without implying a general robustness or identifiability result.

\subsubsection{Complete Bi-HYCO implementation details}\label{supp:elliptic-hyco-implementation}\label{app:hyco_details}

This subsection summarizes the numerical and algorithmic settings used in the experiments of Section~\ref{sec:elliptic}. All computations were performed with a fixed random seed (SEED=1240) to ensure reproducibility. Further implementation details are available in the accompanying source code.

Table~\ref{tab:hyco_params} lists the main hyperparameters of the Bi-HYCO framework. The neural network used by the synthetic component is a fully connected feedforward network with 4 hidden layers of width 96 and \(\tanh\) activation. The physical component uses an unstructured Delaunay mesh with the number of points given in the table. Both local units use the Adam optimizer \cite{kingma2015adam}. The learning rates for the coefficients and network parameters are decayed exponentially every 100 rounds by a factor of 0.9995. The coefficients are parametrized via a softplus transformation to enforce positivity. Interaction points are resampled every round with a mixture of 70\% near-interface and 30\% uniform samples.

\begin{table}[htbp]
	\centering
	\begin{tabular}{@{}l@{\quad}c@{}}
		\toprule
		Parameter & Value \\
		\midrule
		Communication rounds \(M\) & 20000 \\
		Local updates per round \(E_{\text{local}}\) & 2 \\
		Interaction points \(H\) & 1000 \\
		Physical observations & 800 \\
		Synthetic observations & 800 \\
		Mesh interior points (coarse) & 900 \\
		Mesh interior points (fine) & 1800 \\
		Mesh outer boundary points (coarse) & 180 \\
		Mesh outer boundary points (fine) & 360 \\
		Mesh interface points (coarse) & 160 \\
		Mesh interface points (fine) & 320 \\
		NN width / depth & 96 / 4 \\
		Initial LR for coefficients & \(5\times10^{-3}\) \\
		Initial LR for network & \(8\times10^{-4}\) \\
		Physical observation weight \(\alpha_{\text{phy}}\) & 0.3 \\
		Synthetic observation weight \(\alpha_{\text{syn}}\) & 0.2 \\
		Regularization \(\beta_{\text{phy}}\) & \(10^{-8}\) \\
		Weight decay \(\beta_{\text{syn}}\) & \(10^{-8}\) \\
		Scalarization weights (default) & $(\omega_{\rm phy},\omega_{\rm syn})=(0.5,0.5)$ \\
		Coordinator aggregation weights & $(w_1,w_2)=(0.5,0.5)$\\
		$\lambda_{\rm int}$ & $1.0$\\
		\bottomrule
	\end{tabular}
	\caption{Key hyperparameters for Decentralized Bi-HYCO.}
	\label{tab:hyco_params}
\end{table}

For the vanishing-interaction experiment in Section \ref{subsec:elliptic-vanishing-interaction}, all the settings reported in Table \ref{tab:hyco_params} are retained, except that $\lambda_{\rm int}$ is set from $1$ to $0$.

\subsubsection{PINN and XPINN implementation details}
\label{supp:elliptic-pinn-implementation}

For comparison, we optimized standard PINN and XPINN under identical data and evaluation conditions. The networks have the same architecture as the synthetic component (width 96, depth 4, \(\tanh\) activation), except XPINN uses two independent networks of width 80 for each subdomain. Both are optimized for 30000 epochs with Adam at a fixed learning rate of \(10^{-3}\). The loss functions include data fitting, PDE residual, boundary conditions, and interface continuity (solution and flux), with weights chosen to balance the terms. The coefficients are parametrized in the same softplus form as in Bi-HYCO.

Table~\ref{tab:pinn_params} summarizes the loss weights.

\begin{table}[htbp]
	\centering
	\begin{tabular}{@{}l@{\quad}c@{\quad}c@{}}
		\toprule
		Loss component & PINN & XPINN \\
		\midrule
		Data & 10.0 & 10.0 \\
		PDE residual & 1.0 & 1.0 \\
		Boundary (Dirichlet) & 10.0 & 10.0 \\
		Interface solution jump & 5.0 & 10.0 \\
		Interface flux jump & 1.0 & 2.0 \\
		\bottomrule
	\end{tabular}
	\caption{Loss weights for PINN and XPINN (both at \(\eta=0.00\)).}
	\label{tab:pinn_params}
\end{table}

All other settings (mesh, observations, noise generation, evaluation metrics) are identical to those used for Bi-HYCO. The optimization and evaluation meshes are the same as described in Section~\ref{sec:elliptic}.

\subsection{Navier-Stokes experiment}
\label{supp:ns-additional}

This section summarizes the numerical and algorithmic settings used in the experiments of Section~\ref{sec:annular-ns}. All computations were performed with a fixed random seed (SEED=7) to ensure reproducibility. Further implementation details are available in the accompanying source code.

\subsubsection{Vanishing-interaction experiment}
\label{supp:ns-vanishing-interaction}\label{subsec:ns-vanishing-interaction}

To isolate the effect of state-level interaction from that of coordinator aggregation, we repeat the experiment at the highest noise level, \(\eta=0.20\), after setting \(\lambda_{\rm int}=0\). All other elements of the fragmented-observation procedure, including the data distribution, local optimization, and coordinator aggregation, are left unchanged. We compare this aggregation-only configuration with the full Bi-HYCO reconstruction obtained with \(\lambda_{\rm int}>0\), together with the ground-truth parameter values.

Table \ref{tab:ns-interaction-ablation} shows a substantial deterioration in parameter recovery when the interaction term is suppressed. The effect is particularly pronounced for the viscosity, while the forcing-amplitude estimate also becomes less accurate. Overall, the joint parameter error increases markedly in the aggregation-only configuration.

\begin{table}[htbp]
	\centering
	\small
	\begin{tabular}{lcc}
		\toprule
		Metric
		& Full Bi-HYCO
		& Vanishing interaction, $\lambda_{\rm int}=0$
		\\
		\midrule
		$\widehat{\nu}$
		& $1.155\times10^{-3}$
		& $1.920\times10^{-3}$
		\\
		$\widehat{\rho}$
		& $0.99864$
		& $1.04749$
		\\
		$\delta_{\nu}$
		& $0.15532$
		& $0.92009$
		\\
		$\delta_{\rho}$
		& $0.00136$
		& $0.04749$
		\\
		$e_{\gamma}$
		& $0.15532$
		& $0.92131$
		\\
		\bottomrule
	\end{tabular}
	\caption{Parameter recovery for the annular Navier--Stokes inverse problem
		at $\eta=0.20$. The Bi-HYCO values are those reported in
		Table~\ref{tab:hyco-noise-param}.}
	\label{tab:ns-interaction-ablation}
\end{table}

This matched comparison indicates that, in this noisy realization, coordinator aggregation alone is not sufficient to obtain the accuracy achieved by the complete method. The explicit agreement between the physical and synthetic states therefore provides a significant contribution to parameter recovery. As this comparison concerns a single noise realization and parameter configuration, it should nevertheless be interpreted as numerical evidence rather than as a general robustness statement.

\subsubsection{Bi-HYCO implementation details}
\label{supp:ns-hyco-implementation}\label{app:annular-hyco}

Table~\ref{tab:annular_hyco_params} lists the main hyperparameters of the Bi-HYCO framework. The synthetic component uses a fully connected neural network with Fourier feature embedding, 4 hidden layers of width 80 and \(\tanh\) activation. The physical component solves the PDE on a polar grid with \(64\) radial and \(128\) angular cells, using the same second-order scheme as the reference solver but with a coarser resolution and larger time step. The physical parameters are optimized with L-BFGS-B (maximum 20 iterations per round), while the synthetic network parameters are optimized with Adam \cite{kingma2015adam} at a fixed learning rate of \(10^{-3}\). Interaction points are resampled uniformly at random in space and time every communication round.

\begin{table}[htbp]
	\centering
	\begin{tabular}{@{}l@{\quad}c@{}}
		\toprule
		Parameter & Value \\
		\midrule
		Communication rounds \(M\) & 150 \\
		Local updates per round \(E_{\text{local}}\) & 100 \\
		Interaction points \(H\) & 3000 \\
		Physical observations & 2500 \\
		Synthetic observations & 2500 \\
		Validation points & 3000 \\
		PDE solver grid \((n_r \times n_\theta)\) & \(64 \times 128\) \\
		Reference solver grid \((n_r \times n_\theta)\) & \(128 \times 256\) \\
		Time step (PDE solver) \(dt\) & \(5\times10^{-3}\) \\
		Time step (reference) \(dt_{\text{fine}}\) & \(3.5\times10^{-3}\) \\
		Number of stored snapshots & 13 \\
		NN width / depth & 80 / 4 \\
		%Fourier feature frequencies (spatial) & \(\{1,2,3,5\}\) \\
		%Fourier feature frequencies (temporal) & \(\{1,2,4\}\) \\
		Initial LR for network & \(10^{-3}\) \\
		L-BFGS-B max iterations & 20 \\
		Interaction weight \(\lambda_{\rm int}\) & 2.0 \\
		Effective Physical data weight \(a_{\rm phy}=w_{\text{data}}^{\text{phy}}\) & 1.0 \\
		Effective Synthetic data weight \(a_{\rm syn}=w_{\text{data}}^{\text{syn}}\) & 4.0 \\
		%IC weight (synthetic) \(w_{\text{ic}}^{\text{syn}}\) & 2.0 \\
		%BC weight (synthetic) \(w_{\text{bc}}^{\text{syn}}\) & 0.5 \\
		%Aggregation scheme & Arithmetic averaging (equal weights) \\
		\bottomrule
	\end{tabular}
	\caption{Key hyperparameters for the decentralized Bi-HYCO configuration in the annular experiment.}
	\label{tab:annular_hyco_params}
\end{table}

For the vanishing-interaction experiment in Section \ref{subsec:ns-vanishing-interaction}, all the settings reported in Table \ref{tab:annular_hyco_params} are retained, except that $\lambda_{\rm int}$ is set from $2$ to $0$.

\subsubsection{Initialization}
\label{supp:ns-initialization}

For the inversion of the physical parameters $(\nu,\rho)$, all three methods---Bi-HYCO, PINN, and XPINN---are subjected to the same feasible box constraints
\[
\nu \in [2\times 10^{-4},\, 5\times 10^{-3}], \qquad
\rho \in [0.25,\, 1.75].
\]
The nominal parameter initialization is \((\nu_0,\rho_0)=(0.5,0.5)\) for all methods. For Bi-HYCO, the bounded L-BFGS-B solver automatically projects this nominal value onto the feasible box, so the effective starting point is \((\nu_0,\rho_0)=(5\times 10^{-3},0.5)\). For PINN and XPINN, the feasible set is enforced by projecting the parameters onto the box after each gradient step, yielding the same effective initialization. This consistency across methods ensures a fair comparison of the subsequent optimization dynamics.

\subsubsection{PINN and XPINN implementation details}
\label{supp:ns-pinn-implementation}\label{app:annular-pinn}

For comparison, we optimized standard PINN and XPINN under identical data and evaluation conditions. The networks have the same architecture as the synthetic component (Fourier feature embedding, width 80, depth 4, \(\tanh\) activation), except XPINN uses four independent networks (one per angular sector) each with the same architecture. Both comparison methods are optimized for 10000 epochs with Adam \cite{kingma2015adam} at a learning rate \(8\times10^{-4}\) and a stepwise scheduler (decay by 0.75 every 250 epochs). The loss functions include data fitting, PDE residual, Poisson equation residual, initial condition, boundary condition (\(\psi=0\) and \(\omega=0\) on both walls, consistent with the homogeneous wall-vorticity condition used in the physical solver), and XPINN interface-continuity conditions for \(\psi\) and \(\omega\) across sector boundaries. Table~\ref{tab:annular_pinn_params} summarizes the loss weights.

\begin{table}[htbp]
	\centering
	\begin{tabular}{@{}l@{\quad}c@{\quad}c@{}}
		\toprule
		Loss component & PINN & XPINN \\
		\midrule
		Data fitting & 10.0 & 10.0 \\
		PDE residual & 1.0 & 1.0 \\
		Poisson residual & 1.0 & 1.0 \\
		Initial condition & 3.0 & 3.0 \\
		Boundary condition (walls) & 1.0 & 1.0 \\
		Interface continuity & --- & 2.0 \\
		\bottomrule
	\end{tabular}
	\caption{Loss weights for PINN and XPINN (annular experiment, \(\eta=0.00\)).}
	\label{tab:annular_pinn_params}
\end{table}

All other settings (observation points, validation set, noise generation, evaluation metrics, reference solution grid) are identical to those used for Bi-HYCO. The optimization and evaluation grids are the same as described in Section~\ref{sec:annular-ns}.

}
\bibliographystyle{acm}
\bibliography{References}

@book{alexanderian2026computational,
  title     = {Computational Inverse Problems Governed by {PDEs}},
  author    = {Alexanderian, Alen},
  series    = {Other Titles in Applied Mathematics},
  year      = {2026},
  publisher = {Society for Industrial and Applied Mathematics},
  address   = {Philadelphia, PA},
  pages     = {336},
  isbn      = {978-1-61197-881-0},
  doi       = {10.1137/1.9781611978827},
  note      = {Electronic ISBN: 978-1-61197-882-7; SIAM Book Code OT211}
}

@article{AttouchBolteSvaiter2013,
  author  = {Hedy Attouch and J{\'e}r{\^o}me Bolte and Benar Fux Svaiter},
  title   = {Convergence of Descent Methods for Semi-Algebraic and Tame Problems: Proximal Algorithms, Forward--Backward Splitting, and Regularized Gauss--Seidel Methods},
  journal = {Mathematical Programming},
  volume  = {137},
  number  = {1--2},
  pages   = {91--129},
  year    = {2013},
  doi     = {10.1007/s10107-011-0484-9}
}

@book{banks1989estimation,
  title     = {Estimation Techniques for Distributed Parameter Systems},
  author    = {Banks, H. T. and Kunisch, Karl},
  series    = {Systems \& Control: Foundations \& Applications},
  volume    = {1},
  year      = {1989},
  publisher = {Birkh{\"a}user},
  address   = {Boston, MA},
  isbn      = {978-0-8176-3433-9},
  doi       = {10.1007/978-1-4612-3700-6},
}

@article{BolteDaniilidisLewis2007,
  author  = {J{\'e}r{\^o}me Bolte and Aris Daniilidis and Adrian Lewis},
  title   = {The {\L}ojasiewicz Inequality for Nonsmooth Subanalytic Functions with Applications to Subgradient Dynamical Systems},
  journal = {SIAM Journal on Optimization},
  volume  = {17},
  number  = {4},
  pages   = {1205--1223},
  year    = {2007},
  doi     = {10.1137/050644641}
}

@article{BolteSabachTeboulle2014,
  author  = {J{\'e}r{\^o}me Bolte and Shoham Sabach and Marc Teboulle},
  title   = {Proximal Alternating Linearized Minimization for Nonconvex and Nonsmooth Problems},
  journal = {Mathematical Programming},
  volume  = {146},
  number  = {1--2},
  pages   = {459--494},
  year    = {2014},
  doi     = {10.1007/s10107-013-0701-9}
}

@book{chavent2010nonlinear,
  title     = {Nonlinear Least Squares for Inverse Problems:
               Theoretical Foundations and Step-by-Step Guide for Applications},
  author    = {Chavent, Guy},
  series    = {Scientific Computation},
  edition   = {1},
  year      = {2010},
  publisher = {Springer},
  address   = {Dordrecht},
  pages     = {xiv+360},
  isbn      = {978-90-481-2785-6},
  doi       = {10.1007/978-90-481-2785-6},
}

@book{Clarke1990,
  author    = {Frank H. Clarke},
  title     = {Optimization and Nonsmooth Analysis},
  publisher = {Society for Industrial and Applied Mathematics},
  address   = {Philadelphia},
  year      = {1990},
  series    = {Classics in Applied Mathematics},
  volume    = {5},
  isbn      = {9780898712568}
}

@book{dautray1988mathematical,
  author    = {Dautray, Robert and Lions, Jacques-Louis},
  title     = {Mathematical Analysis and Numerical Methods for Science and Technology},
  subtitle  = {Volume 2: Functional and Variational Methods},
  volume    = {2},
  publisher = {Springer-Verlag},
  address   = {Berlin},
  year      = {1988},
  isbn      = {978-3-540-19045-5},
  note      = {With the collaboration of Michel Artola; translated from the French by Ian N. Sneddon}
}

@book{Jahn2004,
  author    = {Jahn, Johannes},
  title     = {Vector Optimization: Theory, Applications, and Extensions},
  publisher = {Springer},
  address   = {Berlin},
  year      = {2004}
}

@article{giles2000introduction,
  title     = {An Introduction to the Adjoint Approach to Design},
  author    = {Giles, Michael B. and Pierce, Niles A.},
  journal   = {Flow, Turbulence and Combustion},
  volume    = {65},
  number    = {3--4},
  pages     = {393--415},
  year      = {2000},
  month     = dec,
  publisher = {Kluwer Academic Publishers},
  doi       = {10.1023/A:1011430410075},
}

@misc{Github,
  author = {Biccari, Umberto and Chen, Jun and Morales, Roberto and Zuazua, Enrique},
  title = {CoDeFeL-HYCO-Bi-objective-perspective},
  year = {2026},
  howpublished = {\url{https://github.com/DeustoTech/CoDeFeL-HYCO-Bi-objective-perspective}},
}

@book{gunzburger2003perspectives,
  title     = {Perspectives in Flow Control and Optimization},
  author    = {Gunzburger, Max D.},
  series    = {Advances in Design and Control},
  volume    = {5},
  year      = {2003},
  publisher = {Society for Industrial and Applied Mathematics},
  address   = {Philadelphia, PA},
  pages     = {xiv+261},
  isbn      = {978-0-89871-527-9},
  doi       = {10.1137/1.9780898718720},
  note      = {Electronic ISBN: 978-0-89871-872-0; SIAM Book Code DC05}
}

@book{hinze2009optimization,
  title     = {Optimization with {PDE} Constraints},
  author    = {Hinze, Michael and Pinnau, Ren{\'e} and
               Ulbrich, Michael and Ulbrich, Stefan},
  series    = {Mathematical Modelling: Theory and Applications},
  volume    = {23},
  edition   = {1},
  year      = {2009},
  publisher = {Springer},
  address   = {Dordrecht},
  pages     = {xii+270},
  isbn      = {978-1-4020-8838-4},
  doi       = {10.1007/978-1-4020-8839-1},
}

@inproceedings{kingma2015adam,
  title         = {Adam: A Method for Stochastic Optimization},
  author        = {Kingma, Diederik P. and Ba, Jimmy},
  booktitle     = {Proceedings of the 3rd International Conference on
                   Learning Representations ({ICLR})},
  year          = {2015},
  address       = {San Diego, CA, USA},
  eprint        = {1412.6980},
  archiveprefix = {arXiv},
  primaryclass  = {cs.LG}
}

@book{kirsch2021introduction,
  title     = {An Introduction to the Mathematical Theory of Inverse Problems},
  author    = {Kirsch, Andreas},
  series    = {Applied Mathematical Sciences},
  volume    = {120},
  edition   = {3},
  year      = {2021},
  publisher = {Springer},
  address   = {Cham},
  isbn      = {978-3-030-63342-4},
  doi       = {10.1007/978-3-030-63343-1},
  note      = {Electronic ISBN: 978-3-030-63343-1}
}

@article{liverani2025hyco,
  title={{HYCO}: Hybrid-cooperative learning for data-driven {PDE} modeling},
  author={Liverani, Lorenzo and Steynberg, Matthys and Zuazua, Enrique},
  journal={arXiv preprint arXiv:2509.14123},
  year={2025}
}

@article{liverani2026hyco,
  title={{HYCO}: A Formalism for Hybrid-Cooperative {PDE} Modelling},
  author={Liverani, Lorenzo and Zuazua, Enrique},
  journal={arXiv preprint arXiv:2602.23859},
  year={2026}
}

@inproceedings{li2019convergence,
  title     = {On the Convergence of {FedAvg} on Non-{IID} Data},
  author    = {Li, Xiang and Huang, Kaixuan and Yang, Wenhao and Wang, Shusen and Zhang, Zhihua},
  booktitle = {International Conference on Learning Representations},
  year      = {2020}
}

@inproceedings{mcmahan2017communication,
  title     = {Communication-Efficient Learning of Deep Networks from Decentralized Data},
  author    = {McMahan, Brendan and Moore, Eider and Ramage, Daniel and
               Hampson, Seth and {Ag{\"u}era y Arcas}, Blaise},
  editor    = {Singh, Aarti and Zhu, Jerry},
  booktitle = {Proceedings of the 20th International Conference on
               Artificial Intelligence and Statistics},
  series    = {Proceedings of Machine Learning Research},
  volume    = {54},
  pages     = {1273--1282},
  year      = {2017},
  month     = apr,
  publisher = {PMLR},
  address   = {Fort Lauderdale, Florida, USA},
  url       = {https://proceedings.mlr.press/v54/mcmahan17a.html}
}

@book{miettinen1999nonlinear,
  author    = {Kaisa Miettinen},
  title     = {Nonlinear Multiobjective Optimization},
  series    = {International Series in Operations Research \& Management Science},
  volume    = {12},
  year      = {1999},
  publisher = {Springer},
  address   = {New York, NY},
  isbn      = {978-0-7923-8278-2},
  doi       = {10.1007/978-1-4615-5563-6}
}

@article{plessix2006review,
  title     = {A Review of the Adjoint-State Method for Computing the
               Gradient of a Functional with Geophysical Applications},
  author    = {Plessix, R.-E.},
  journal   = {Geophysical Journal International},
  volume    = {167},
  number    = {2},
  pages     = {495--503},
  year      = {2006},
  month     = nov,
  doi       = {10.1111/j.1365-246X.2006.02978.x},
}

@article{quarteroni2025combining,
  title        = {Combining physics-based and data-driven models: advancing the frontiers of research with Scientific Machine Learning},
  author       = {Quarteroni, Alfio and Gervasio, Paola and Regazzoni, Francesco},
  journal      = {Mathematical Models and Methods in Applied Sciences},
  volume       = {35},
  number       = {4},
  pages        = {905--1071},
  year         = {2025},
  doi          = {10.1142/S0218202525500125}
}

@book{tarantola2005inverse,
  title     = {Inverse Problem Theory and Methods for Model Parameter Estimation},
  author    = {Tarantola, Albert},
  series    = {Other Titles in Applied Mathematics},
  year      = {2005},
  publisher = {Society for Industrial and Applied Mathematics},
  address   = {Philadelphia, PA},
  pages     = {ix+339},
  isbn      = {978-0-89871-572-9},
  doi       = {10.1137/1.9780898717921},
  note      = {Electronic ISBN: 978-0-89871-792-1}
}

@article{xu2022physics,
  title   = {Physics Constrained Learning for Data-Driven Inverse
             Modeling from Sparse Observations},
  author  = {Xu, Kailai and Darve, Eric},
  journal = {Journal of Computational Physics},
  volume  = {453},
  pages   = {110938},
  year    = {2022},
  doi     = {10.1016/j.jcp.2021.110938},
}

@article{yang2021bpinns,
  title   = {{B-PINNs}: Bayesian Physics-Informed Neural Networks for
             Forward and Inverse {PDE} Problems with Noisy Data},
  author  = {Yang, Liu and Meng, Xuhui and Karniadakis, George Em},
  journal = {Journal of Computational Physics},
  volume  = {425},
  pages   = {109913},
  year    = {2021},
  doi     = {10.1016/j.jcp.2020.109913},
}

@article{moseley2023fbpinns,
  title        = {Finite Basis Physics-Informed Neural Networks ({FBPINN}s): A Scalable Domain Decomposition Approach for Solving Differential Equations},
  author       = {Moseley, Ben and Markham, Andrew and Nissen-Meyer, Tarje},
  journal      = {Advances in Computational Mathematics},
  volume       = {49},
  number       = {4},
  articleno    = {62},
  year         = {2023},
  month        = jul,
  publisher    = {Springer},
  doi          = {10.1007/s10444-023-10065-9},
}

@article{montans2019data,
  title        = {Data-driven modeling and learning in science and engineering},
  author       = {Mont{\'a}ns, Francisco J. and Chinesta, Francisco and G{\'o}mez-Bombarelli, Rafael and Kutz, J. Nathan},
  journal      = {Comptes Rendus M{\'e}canique},
  volume       = {347},
  number       = {11},
  pages        = {845--855},
  year         = {2019},
  doi          = {10.1016/j.crme.2019.11.009}
}

@article{raissi2019physics,
  title        = {Physics-informed neural networks: A deep learning framework for solving forward and inverse problems involving nonlinear partial differential equations},
  author       = {Raissi, Maziar and Perdikaris, Paris and Karniadakis, George Em},
  journal      = {Journal of Computational Physics},
  volume       = {378},
  pages        = {686--707},
  year         = {2019},
  doi          = {10.1016/j.jcp.2018.10.045}
}

@article{yang2019federated,
  title={Federated machine learning: Concept and applications},
  author={Yang, Qiang and Liu, Yang and Chen, Tianjian and Tong, Yongxin},
  journal={ACM Transactions on Intelligent Systems and Technology},
  volume={10},
  number={2},
  pages={1--19},
  year={2019},
  publisher={ACM New York, NY, USA}
}

@misc{zuazua2026coercivitygap,
  title        = {The Coercivity Gap in Neural {PDE} Solvers: Parameter Escape and Functional Convergence},
  author       = {Zuazua, Enrique},
  howpublished = {To appear in \emph{Mathematical Models and Methods in Applied Sciences}},
  eprint       = {2606.04018},
  archivePrefix = {arXiv},
  primaryClass = {math.NA}
}

@article{jagtap2020extended,
  title={Extended physics-informed neural networks ({XPINNs}): A generalized space-time domain decomposition based deep learning framework for nonlinear partial differential equations},
  author={Jagtap, Ameya D. and Karniadakis, George Em},
  journal={Communications in Computational Physics},
  volume={28},
  number={5},
  year={2020},
  publisher={Brown Univ., Providence, RI (United States)}
}

\end{document}